\documentclass{article}

\PassOptionsToPackage{numbers, compress}{natbib}

\usepackage[hidelinks]{hyperref}
\usepackage{url}
\usepackage{xspace}
\usepackage{listings}

\usepackage[utf8]{inputenc} 
\usepackage[T1]{fontenc}    
\usepackage{booktabs}       
\usepackage{amsfonts}       
\usepackage{nicefrac}       
\usepackage{microtype}      
\usepackage{xcolor}         
\usepackage{enumitem}
\usepackage{graphicx}
\usepackage{subfigure}
\usepackage{amsmath}
\usepackage{amssymb}
\usepackage{mathtools}
\usepackage{amsthm}
\usepackage[capitalize,noabbrev]{cleveref}
\usepackage{multirow}
\usepackage{wrapfig}
\usepackage{subcaption}
\usepackage[misc]{ifsym}
\usepackage{algorithmic}
\usepackage{makecell}

\usepackage{needspace}

\usepackage[most]{tcolorbox}
\newtcolorbox{promptbox}[1]{ 
  colback=gray!5,             
  colframe=black,             
  boxsep=5pt,                 
  arc=4pt,                    
  top=4pt, bottom=4pt, left=4pt, right=4pt, 
  title={#1},                 
  fonttitle=\bfseries,        
  breakable,
}

\usepackage{verbatim}

\usepackage{fancyvrb}
\usepackage{xcolor}

\usepackage[preprint]{neurips_2026}

\usepackage{amsmath}
\usepackage{amssymb}
\usepackage{mathtools}
\usepackage{amsthm}

\def \MethodName {PROCO\xspace}
\def \MethodNameLong {Model-Based \textbf{Pro}active \textbf{Co}st Generation\xspace}

\usepackage[capitalize,noabbrev]{cleveref}

\theoremstyle{plain}
\newtheorem{theorem}{Theorem}[section]
\newtheorem{proposition}[theorem]{Proposition}
\newtheorem{lemma}[theorem]{Lemma}

\theoremstyle{definition}
\newtheorem{definition}[theorem]{Definition}
\newtheorem{assumption}[theorem]{Assumption}
\theoremstyle{remark}

\usepackage[textsize=tiny]{todonotes}

\title{Model-Based Proactive Cost Generation for Learning Safe Policies Offline with Limited Violation Data}

\author{%
  Ruiqi Xue\textsuperscript{1,2}, Lei Yuan\textsuperscript{1,2,3}\thanks{Corresponding Author}, Kainuo Cheng\textsuperscript{1,2}, Jing-Wen Yang\textsuperscript{4}, Yang Yu\textsuperscript{1,2,3}\footnotemark[1]\\
  \textsuperscript{1} National Key Laboratory for Novel Software Technology, Nanjing University\\
  \textsuperscript{2} School of Artificial Intelligence, Nanjing University \\
  \textsuperscript{3} Polixir Technologies
  \\ Nanjing, China\\
  \textsuperscript{4} Tencent Game AI Center \\
  Shenzhen, China
  \\
  \texttt{\{xuerq, yuanl\}@lamda.nju.edu.cn} \\
  \texttt{231220042@smail.nju.edu.cn} \\
  \texttt{jingwenyang@tencent.com} \\
  \texttt{yuy@nju.edu.cn}
}

\begin{document}
\maketitle

\begin{abstract}
Learning constraint-satisfying policies from offline data without risky online interaction is crucial for safety-critical decision making. Conventional methods typically learn cost value functions from abundant unsafe samples to define safety boundaries and penalize violations. However, in high-stakes scenarios, risky trial-and-error is infeasible, yielding datasets with few or no unsafe samples.
Under this limitation, existing approaches often treat all data as uniformly safe, overlooking safe-but-infeasible states—states that currently satisfy constraints but inevitably violate them within a few steps—leading to deployment failures.
Drawing inspiration from the concept of knowledge-data integration, we leverage large language models (LLMs) to incorporate natural language knowledge into the policy to address this challenge. Specifically, we propose \MethodName, a model-based offline safe reinforcement learning (RL) framework tailored to datasets largely free of violations. \MethodName first learns a dynamics model from offline data and constructs a conservative cost function by grounding natural-language knowledge of unsafe states in LLMs, enabling risk estimation even without observed violations. Using the cost function and learned model, \MethodName performs model-based rollouts to synthesize diverse counterfactual unsafe samples, supporting reliable feasibility identification and feasibility-guided policy learning.
Across a range of Safety-Gymnasium tasks with exclusively safe or minimally risky training data, \MethodName integrates seamlessly with a variety of offline safe RL algorithms and consistently demonstrates reduced constraint violations and improved safety performance compared to both the original methods and other behavior cloning baselines.
\end{abstract}

\section{Introduction}
Safe reinforcement learning (RL), which aims to derive policies that satisfy predefined safety constraints, is crucial for real-world applications such as autonomous driving~\citep{zhang2021safe}, robotic control~\citep{brunke2022safe}, and aligning large language models (LLM) with human values~\citep{daisafe}. A fundamental limitation of conventional online RL is its reliance on trial-and-error exploration, which is inherently risky and can preclude its deployment in safety-critical settings~\citep{levine2020offline}. To mitigate these risks, offline safe RL has emerged as a promising paradigm, seeking to learn safe policies exclusively from pre-collected datasets without necessitating hazardous online interactions~\citep{liu2023datasets,zhengsafe}.

%

Recently, numerous methods have been proposed for offline safe policy learning. One line focuses on learning conservative cost value functions to penalize unsafe actions~\cite{xu2022constraints, guo2025constraint}, while another leverages cost advantage functions to weight behavior cloning (BC) on offline data~\cite{leecoptidice, zhengsafe, koiralalatent, chemingui2025constraint}. Both approaches critically rely on abundant unsafe samples for accurate cost estimation and effective policy optimization.
However, in many safety-critical high-stakes scenarios, such as autonomous driving or robotics manipulation, collecting large numbers of unsafe samples is impractical, as it is likely to cause detrimental effects on the agent or the environment~\cite{dulac2019challenges}, yielding dataset with only few or no unsafe samples.
In such cases, existing methods may fail to enforce safety, as they treat all samples equally safe and overlook \textbf{safe-yet-infeasible} states.
For instance, during robot data collection, to avoid harming the agent, external interventions are used to stop the process whenever the robot approaches a collision with an obstacle, leaving terminal states that are currently safe but would soon violate constraints due to inertia if no external intervention is applied~\citep{bansal2017hamilton}.
This limitation thus motivates a new problem: \textbf{how can we learn a safe policy offline when unsafe samples are scarce or entirely absent?} 

The root cause of the aforementioned challenge is, due to the absence of unsafe samples, the dataset lacks a crucial piece of knowledge—namely, which behaviors are unsafe or risky. This observation motivates a knowledge-data integration approach, wherein human or external knowledge is provided to the agent to bridge this gap~\cite{nie2025data}. Given that such knowledge typically exists in natural language~\cite{alismail2025survey}, LLMs—with their robust language processing and reasoning capabilities—serve as an ideal vehicle for knowledge processing. Furthermore, to align the LLM-derived knowledge with the empirical data, we incorporate the concept of model rollouts. Specifically, we propose \MethodNameLong (\MethodName), a model-based offline safe RL algorithm that identifies infeasible states and learns safe policies from datasets containing few or no unsafe samples. \MethodName first leverages an LLM to derive a conservative cost function from the natural language knowledge of the constraint, which is then validated and refined using the available safe data and, when present, a limited number of unsafe samples. Subsequently, it learns a dynamics model from the offline dataset. Using the cost function and dynamics model, \MethodName simulates future state evolution to generate diverse and informative unsafe samples, enabling efficient detection of infeasible states. Moreover, the conservative cost function enables states in close proximity to unsafe ones to be labeled as unsafe as well, shortening the transition steps from infeasible to unsafe states and thus reducing the influence of model errors on infeasible state identification.
Extensive experiments across diverse Safety-Gymnasium environments demonstrate that, when integrated with multiple offline safe RL algorithms, \MethodName significantly outperforms both the original algorithms and other behavior cloning–based baselines on safe-only datasets, achieving improvements of over 400\% in safety performance.

\section{Related Work}

\paragraph{Safe RL.}
Safe RL seeks to maximize reward while ensuring safety, commonly modeled as a Constrained Markov Decision Process (CMDP)~\citep{altman2021constrained} and solved via constrained optimization~\citep{garcia2015comprehensive,gu2024review}.
A widely adopted approach is Lagrangian-based methods, which learn a cost value function with an adaptive multiplier to enforce safety~\citep{stooke2020responsive}.
For hard, state-wise constraints, Hamilton–Jacobi (HJ) reachability analysis~\citep{bansal2017hamilton} has been applied to learn cost functions and enforce stricter safety~\citep{yu2022reachability,ganai2023iterative}.
However, both approaches require unsafe real-world interactions, which limit their practicality. To overcome this limitation, recent research has shifted toward offline safe RL, learning safe policies from pre-collected data to avoid unsafe exploration.
Certain approaches integrate conservative value estimation into cost value function learning to counteract cost value underestimation~\citep{xu2022constraints,guo2025constraint}, whereas others employ BC-based policy learning, including Decision Transformer~\citep{liu2023constrained,xueadaptable}, DICE~\citep{leecoptidice}, and IQL~\citep{zhengsafe,koiralalatent,zifan2025c2iql}, to more effectively address extrapolation errors~\citep{fujimoto2019off} in safe RL.

\paragraph{LLM-assisted decision-making.}
Leveraging the powerful information processing and reasoning capabilities of LLMs, they have recently been widely adopted in RL~\citep{cao2024survey}. One prominent line of work explores the use of LLMs for direct decision-making by generating actions or high-level plans conditioned on observations~\citep{yao2023react,shinn2023reflexion,prasad2023adapt}. Yet, these methods are generally applicable only to high-level, highly abstract decision-making tasks.
A more broadly adopted line of work in RL involves reward function generation, in which LLMs are employed to directly construct reward functions that support skill discovery~\citep{yulanguage}, policy learning~\citep{song2023self,maeureka,xietext2reward}, exploration~\citep{triantafyllidis2024intrinsic}, or teammate generation~\citep{li2025llm}. However, the use of LLMs for cost generalization in safe RL remains largely underexplored.
More related work can be seen in Appendix~\ref{more related work}.

\section{Preliminaries}
\subsection{Safe RL and Offline Safe RL}
In this work, we focus on safe RL under hard-constraints, which can be modeled as a hard-constraint CMDP, defined as a tuple $\langle S, A, r, h, c, P, \gamma \rangle$. Here, $S$ and $A$ denote the state and action spaces, respectively;
$r: S \times A \rightarrow [-R_{\text{max}}, R_{\text{max}}]$ represents the reward function, $h: S \rightarrow [h_{\text{min}}, h_{\text{max}}]$ is the constraint violation function, and $c: S \rightarrow \{0, 1\}$ is the cost function. $P: S \times A \times S \rightarrow [0, 1]$ specifies the transition dynamics, and $\gamma \in (0, 1)$ is the discount factor.
Typically, $c(s)=\mathbb{I}(h(s)>0)$, which means $h(s)>0$ is unsafe while $h(s)\leq 0$ is safe.
A policy $\pi: S \rightarrow \Delta(A)$ maps states to action distributions.
Under policy $\pi$, the expected discounted reward return and cost return are defined as $R(\pi) = \mathbb{E}_{\tau\sim P_\pi}\left[\sum_{t=0}^\infty \gamma^t r(s_t,a_t)\right]$ and $C(\pi) = \mathbb{E}_{\tau\sim P_\pi}\left[\sum_{t=0}^\infty \gamma^tc(s_t)\right]$, where $\tau = (s_0, a_0, s_1, a_1, \dots)\sim P_\pi$ denotes a trajectory induced by $\pi$ and the environment dynamics $P$. Thus, the objective of solving a hard-constraint CMDP is to find a policy that maximizes reward return while ensuring the cost return remains zero.

In the offline setting, we are given an offline dataset $\mathcal{D}$ generated by the behavior policy $\pi_\beta$. Now, the goal is to learning a safe policy purely from this dataset. To avoid extrapolation error in offline RL, the optimization objective is formulated as:
\begin{equation}
\begin{aligned}
\label{obj_2}
&\max_\pi\ R(\pi), \ s.t.\ C(\pi)\leq 0\ ;\ D(\pi||\pi_\beta)\leq\epsilon,
\end{aligned}
\end{equation}
where $D(\pi||\pi_\beta)$ is a divergence term (e.g., KL divergence $D_{\text{KL}}(\pi||\pi_\beta)$) used to prevent distribution shift.
To investigate the problem of learning safe policies under scarce or absent unsafe samples, we focus on scenarios where $\mathcal{D}$ contains no unsafe samples. Optionally, we assume the availability of an extremely small dataset $\mathcal{D}_{\text{unsafe}}$ consisting of no more than 100 unsafe transitions, satisfying $|\mathcal{D}_{\text{unsafe}}|<<|\mathcal{D}|$.
Meanwhile, in most practical applications, a continuous constraint violation function $h$ is often unavailable. Accordingly, in this work we define $h$ as also a binary function: $h(s)=h_{\text{min}}\leq0$ if $c(s)=0$, and $h(s)=h_{\text{max}}>0$ otherwise.
To compensate for the lack of unsafe samples, we assume the availability of a natural language specification of the task’s safety constraint, $L_{\text{cost}}$, to provide safety-related information. Thus, the agent must exploit both the safe dataset $\mathcal{D}$ and the constraint description $L_{\text{cost}}$ in order to acquire a safe policy.

\subsection{Hamilton-Jacobi Reachability Analysis}
To solve hard-constrained safe RL, HJ reachability analysis~\citep{bansal2017hamilton} is one of the most widely used approaches. It typically models $\pi$ as a deterministic policy and defines the feasible set accordingly.
\begin{definition}[Feasible set \citep{yu2022reachability}]
\label{def_1}
The feasible set of a specific policy $\pi$ can be defined as
\begin{equation}
    S_f^\pi:=\{s\in S|h(s_t^\pi|s_0=s)\leq 0,\forall t\in \mathbb{N}\}.
\end{equation}
The largest feasible set $S_f^*$ is a subset of $S$ composed of states from which there exists at least one policy that keeps the system satisfying the constraint, i.e.,
\begin{equation}
    S_f^*:=\{s\in S|\exists\pi,h(s_t^\pi|s_0=s)\leq0,\forall t\in\mathbb{N}\}.
\end{equation}
\end{definition}
Above, $h(s_t^\pi|s_0=s)$ specifies the state constraint sequence of the trajectory $\{h(s_0^\pi),h(s_1^\pi),\dots,|s_0=s,\pi\}$.
\begin{definition}[Optimal feasible value function \citep{yu2022reachability}]
\label{def_2}
The optimal feasible state-value function $V_h^*$, and the optimal feasible action-value function $Q_h^*$ are defined as
\begin{equation}
\begin{aligned}
    V_h^*(s)&:=\min_\pi V_h^{\pi}(s):=\min_\pi\max_{t\in\mathbb{N}}h(s_t),s_0=s,a_t\sim\pi(\cdot|s_t),\\Q_h^*(s,a)&:=\min_\pi Q_h^{\pi}(s,a):=\min_\pi\max_{t\in\mathbb{N}}h(s_t),s_0=s,a_0=a,a_t\sim\pi(\cdot|s_t),
\end{aligned}
\end{equation}
where $V_h^\pi$ represents the maximum constraint violations in the trajectory induced by policy $\pi$ starting from the state $s$. The (optimal) feasibile value function possesses the following properties \citep{zhengsafe}:
\begin{itemize}[leftmargin=1em]
    \item $V_h^\pi(s)\leq 0\Rightarrow\forall s_t,h(s_t)\leq 0$, indicating $\pi$ can satisfy the hard-constraint starting from $s$. $V_h^*(s)\leq 0\Rightarrow \exists \pi, V_h^\pi(s)\leq 0$, meaning there exists a policy that satisfies the hard-constraint.
    \item The feasible set and largest feasible set can be written as
    \begin{equation}
        S_f^\pi:=\{s|V_h^\pi(s)\leq0\},\ S_f^*:=\{s|V_h^*(s)\leq 0\}.
    \end{equation}
\end{itemize}
\end{definition}
Based on Definition~\ref{def_1}, once $S_f^*$ is obtained, the feasibility of states in $\mathcal{D}$ can be determined. 
Furthermore, Definition~\ref{def_2} indicates that $S_f^*$ can be obtained by computing the optimal feasible value function $V^*_h(s)$. Thus, the feasible Bellman operator $\mathcal{B}^*$~\citep{fisac2019bridging} is proposed:
\begin{equation}
\begin{aligned}
    \mathcal{B}^*Q_h(s,a):=(1-\gamma)h(s)+\gamma\max\{h(s),V_h^*(s')\},\ V_h^*(s')=\min_{a'}Q_h(s',a').
\end{aligned}
\end{equation}
Although the aforementioned $S_f^\pi$ and $V_h^\pi$ are defined for a deterministic policy $\pi$, the computation of $V_h^*$ and $\mathcal{B}^*$ is policy-independent and does not rely on whether $\pi$ is deterministic. Therefore, these quantities can be readily applied to stochastic policies as well~\citep{zhengsafe}.

\section{Method}
\begin{figure*}[t]
    \centering
    \includegraphics[width=0.95\textwidth]{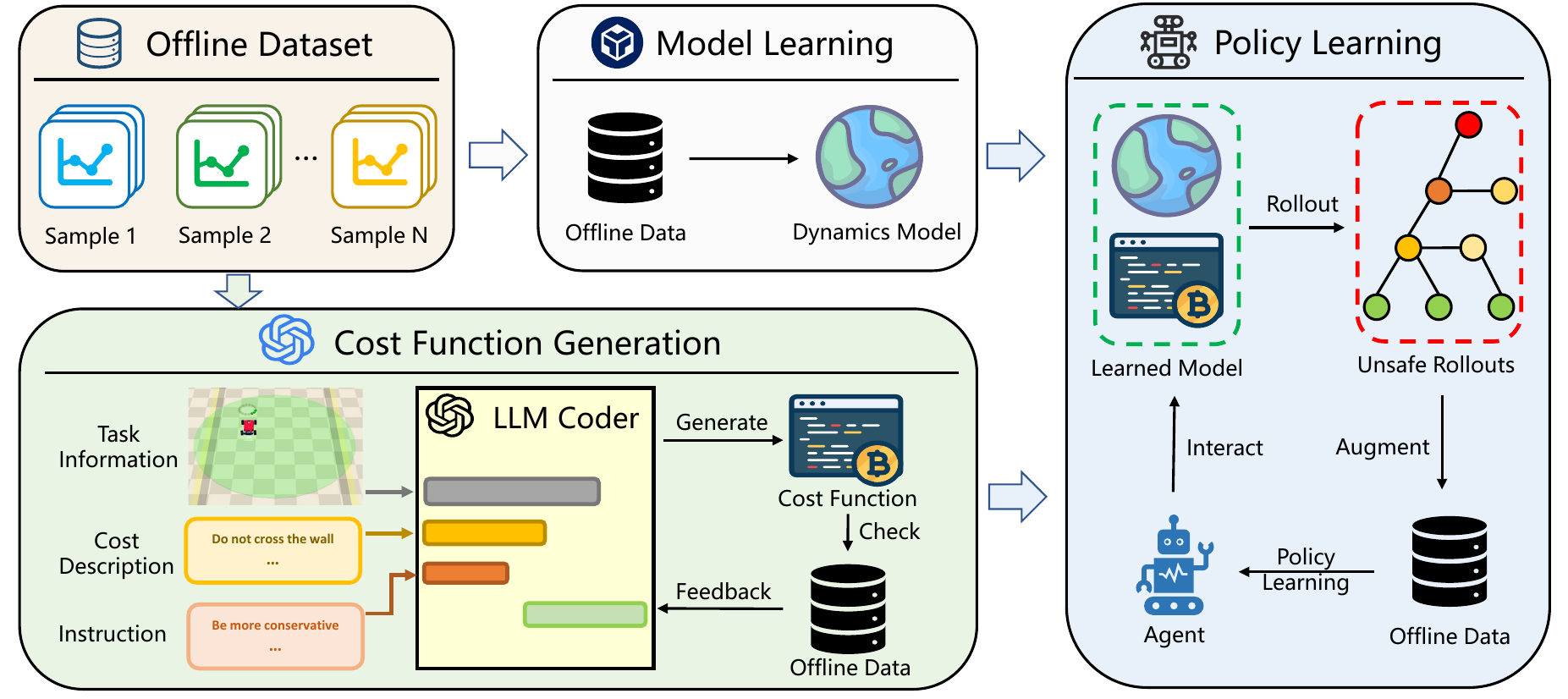}
    \caption{Structure of \MethodName.}
    \label{fig:flow}
\end{figure*}


This section details \MethodName, a novel algorithm for offline safe policy learning when unsafe samples are scarce or absent (Figure~\ref{fig:flow}). Section~\ref{method_1} describes our dynamics model-based feasibility identification, Section~\ref{method_2} presents the motivation and use of LLMs to generate conservative cost functions, and Section~\ref{method_3} outlines the overall pipeline for assisting safer policy learning with \MethodName. Implementation details can be found in Appendix~\ref{imp detail}.

\subsection{Feasibility Identification with a Dynamics Model}
\label{method_1}


Theoretically, an accurate estimate of $V_h^*(s)$ suffices to evaluate the feasibility of state $s$.
However, computing $V^*_h(s)$ via $\mathcal{B}^*$ demands a substantial amount of unsafe samples beyond $\mathcal{D}$. To overcome this limitation, we propose leveraging a learned dynamics model to proactively generate future unsafe samples.
First, we train an ensemble dynamics model $\hat{T}$ using the offline dataset $\mathcal{D}$:
\begin{equation}
    \min_{\hat{T}} \mathbb E_{(s,a,s')\sim \mathcal{D}}[||\hat{T}(s,a)-s'||_2^2].
\end{equation}
Subsequently, leveraging $\hat{T}$ together with the dataset $\mathcal{D}$, we perform branched rollout to simulate future trajectories. To address the potential underestimation of $V^*_h$ due to model uncertainty, we introduce a conservative feasible Bellman operator, denoted as $\bar{\mathcal{B}}^*$:
\begin{equation}
\label{conservative bellman}
    \bar{\mathcal{B}^*}Q_h(s,a):=(1-\gamma)h(s)+\gamma\max\{h(s),\max_{s'\in\hat{T}(s,a)}\min_{a'}Q_h(s',a')\},
\end{equation}
For $\bar{\mathcal{B}}^*$, we establish the following desirable proposition:
\begin{proposition}
\label{pro_1}
    $\bar{\mathcal{B}}^*$ is a $\gamma$ contraction mapping in the $\infty-$norm, and satisfies that $\bar{Q^*_h}(s,a)\geq Q^*_{h,T}(s,a)$ for all $(s,a)$ and for all $T\in\hat{T}$, where $\bar{Q^*_h}(s,a)$ is the convergence result of $\bar{\mathcal{B}^*}$ and $Q^*_{h,T}(s,a)$ is the convergence result of $\mathcal{B}^*$ under transition dynamics $T$.
\end{proposition}
Finally, under certain assumptions, we can guarantee the effectiveness of feasibility identification when employing $\bar{\mathcal{B}}^*$ together with $\hat{T}$:
\begin{assumption}
\label{ass_1}
    There exists a horizon $H^*\in \mathbb{N}$ such that, for any infeasible state $s$, any sequence of actions $a_0,\dots,a_{H^*-1}$ will lead to an unsafe state.
\end{assumption}
\begin{assumption}
\label{ass_2}
    The learned ensemble dynamics model $\hat{T}(s,a)$ is calibrated, that is, for all $(s,a)\in S\times A$, the ground truth dynamics model $T(s,a)$ satisfies that $T(s,a)\in \hat{T}(s,a)$.
\end{assumption}
\begin{theorem}
\label{th_1}
    If Assumption~\ref{ass_1} and Assumption~\ref{ass_2} hold, $\forall s\in \bar{S_f^*}=S-S_f^*,\forall a$, the convergence result $\bar{Q_h^*}$ learned by $\bar{\mathcal{B}^*}$ and $\hat{T}$ rollout data satisfies that, $\forall a,\bar{Q_h^*}(s,a)>0$ for large enough $\gamma$.
\end{theorem}

\subsection{LLM-assisted Conservative Cost Function Generation}
\label{method_2}

Theorem~\ref{th_1} enables feasibility identification for samples in $\mathcal{D}$, yet it critically depends on Assumption~\ref{ass_2}, which requires exact accuracy of the learned dynamics transitions, a condition rarely satisfied in practice. We therefore relax it as Assumption~\ref{ass_3} and introduce another assumption:
\begin{assumption}
\label{ass_3}
    For a given policy $\pi$, there exists $\delta$, $\forall s,a$, $\max_t\mathbb{E}_{s\sim p_1^t(s)}D_{\text{KL}}(p_1(s'|s)||p_2(s'|s))\leq \delta$, where $p_1(s'|s)=T(s'|s,\pi(\cdot|s))$ for $t>0$ and $p_1(s'|s)=T(s'|s,a)$ for $t=0$, and $p_2(s'|s)=T'(s'|s,\pi(\cdot|s))$ for $t>0$ and $p_2(s'|s)=T'(s'|s,a)$ for $t=0$, $T$ is the ground truth dynamics model, and $T'$ refers to any one of the dynamics model in the ensemble dynamics model set $\hat{T}$.
\end{assumption}
\begin{assumption}
\label{ass_4}
    For a given constraint violation function $h$, there exists a constant $K$, such that for any two state distributions $\mu,v$, it satisfies that $|\mathbb{E}_\mu[h(s)]-\mathbb{E}_v[h(s)]|\leq K D_{\text{TV}}(\mu||v)$, where $D_{\text{TV}}$ is the total variance distance.
\end{assumption}
Assumption~\ref{ass_4} assumes a correlation between the expected constraint violation and the state distribution, such that similar state distributions correspond to similar expected constraint violations. This assumption is generally valid in practice, and based on it, we can proceed to analyze the effectiveness of feasibility identification:
\begin{theorem}
\label{th_2}
    If Assumption~\ref{ass_1}, Assumption~\ref{ass_3} and Assumption~\ref{ass_4} hold. $\forall s\in \bar{S_f^*}=S-S_f^*$, $\forall a$, given policy $\pi$, the convergence result $\bar{Q_h^*}$ learned by $\bar{\mathcal{B}^*}$ and branched rollout data of $\pi$ in $\hat{T}$ satisfies
    \begin{equation}
    \begin{aligned}
        \bar{Q_h^*}(s,a)\geq (1-\gamma^{t^*})h_{\text{min}}+\gamma^{t^*}\mathbb{E}_{p_1^{t^*}}[h(s)]-\gamma^{t^*} {t^*}K\delta,\ t^*=\arg\max_{t\in\{1,\dots,H^*\}}[\mathbb{E}_{p_1^t}[h(s)]].
    \end{aligned}
    \end{equation}
\end{theorem}
\textbf{Remark.} Since $\gamma,h_{\text{min}},K$ are outer given of constant, Theorem~\ref{th_2} reveals that the feasible value function of the infeasible state is lower bounded by $t^*, h(s)$ and $\delta$. To ensure safe decision making, we hope $\bar{Q_h^*}(s,a)$ can be large enough. Since $\mathbb{E}_{p_1^t}[h(s)]$ is upper bounded by $h_{\text{max}}$, a practical way is to lower $\delta$ or lower $t^*$.
Here, $\delta$ denotes the model rollout error. Due to cumulative error, shorter rollout horizons lead to lower errors. The term $t^*$ represents the number of steps required to transition from an infeasible state $s$ to its associated unsafe state $s'$. Thus, the distance between $s$ and $s'$ determines $t^*$, with smaller distances yielding smaller values of $t^*$. Hence, the analysis indicates that \textbf{bringing an infeasible state closer to its corresponding unsafe state reduces the rollout error $\delta$ as well as the step horizon $t^*$, resulting in more efficient feasibility identification}.

Therefore, we propose leveraging an LLM to generate a more conservative cost function, which \textbf{marks states near actual unsafe states also as unsafe}, effectively bringing infeasible states closer to unsafe states.
Specifically, we first provide three natural language descriptions, $L_{\text{task}}$, $L_{\text{cost}}$, and $L_{\text{inst}}$, to the LLM, enabling it to generate the task’s cost function $\bar{c}$:
\begin{equation}
    \bar{c} = \text{LLM}(L_{\text{task}}, L_{\text{cost}}, L_{\text{inst}}),
\end{equation}
where $L_{\text{task}}$ is task-related information, such as the meaning of states; $L_{\text{cost}}$ is the provided language description of the safety constraint; and $L_{\text{inst}}$ provides explicit instructions to the LLM, directing it to generate a cost function that is more conservative than the constraint described in $L_{\text{cost}}$.

Nevertheless, the cost function produced directly by the LLM may not be reliable. To mitigate this, we propose a validation-and-feedback mechanism leveraging both the small unsafe dataset $\mathcal{D}_{\text{unsafe}}$ and the safe dataset $\mathcal{D}$.
Concretely, we begin by validating $\bar{c}$ on $\mathcal{D}_{\text{unsafe}}$, requiring 100\% accuracy to guarantee that all unsafe samples are correctly identified, thereby eliminating safety risks. After this criterion is satisfied, we evaluate the proportion of safe samples in $\mathcal{D}$ that are classified as unsafe, which quantifies the conservativeness of $\bar{c}$. If this proportion lies within the hyperparameter-controlled range $[p_{\text{min}},p_{\text{max}}]$, we deem the conservativeness acceptable and adopt $\bar{c}$ as the final cost function.
Otherwise, we construct a feedback description $L_{\text{feed}}$ from the evaluation outcomes and feed it back into the LLM to guide the regeneration of $\bar{c}$:
\begin{equation}
    \bar{c} = \text{LLM}(L_{\text{task}}, L_{\text{cost}}, L_{\text{inst}}, L_{\text{feed}}).
\end{equation}
If no satisfactory $\bar{c}$ is obtained after exhausting the maximum number of LLM queries, we adopt as the final cost function the candidate that attains the highest accuracy on $\mathcal{D}_{\text{unsafe}}$ and exhibits a conservativeness level on $\mathcal{D}$ closest to the interval $[p_{\text{min}},p_{\text{max}}]$.

\subsection{Overall Algorithm}
\label{method_3}
With the learned $\hat{T}$ and the obtained $\bar{c}$, we can now learn a feasible policy based on the safe-only dataset $\mathcal{D}$.
First, we initialize a model rollout dataset $\mathcal{D}_{\hat{T}}$, and use the learning policy $\pi$ to perform branched rollouts of length $H$ within $\hat{T}$, obtaining a set of branch trajectories $\{\tau_i\}_{i=1}^N$. Each trajectory is of the form $\tau_i=(s_0^i,a_0^i,\bar{c}(s_0^i),\dots,s_{H-1}^i,a_{H-1}^i,\bar{c}(s_{H-1}^i))$, where $s_0^i \in \mathcal{D}$, $s_t^i \sim \hat{T}(\cdot | s_{t-1}^i, a_{t-1}^i)$ for $t>0$.
Next, for any trajectory $\tau_i$, if it contains a safety violation, i.e., $\sum_{t=0}^{H-1}\bar{c}(s_t^i) > 0$, we add it to $\mathcal{D}_{\hat{T}}$. Otherwise, trajectories without safety violations are discarded.
To enhance data diversity during rollouts and thereby improve the efficiency of feasibility identification, we inject Gaussian noise with standard deviation $\sigma_{\text{exp}}$ (a hyperparameter) to the policy actions during rollout.
Meanwhile, the existing samples in $\mathcal{D}$ are further re-labeled with costs using $\bar{c}$. Finally, we derive the constraint violation labels according to $h(s)=h_{\text{min}}\leq 0$ if $\bar{c}(s)=0$, and $h(s)=h_{\text{max}}>0$ otherwise.

With the constraint violation labels in place, we approximate the minimization operator in the (conservative) feasible Bellman update, i.e., $V^*_h(s')=\min_{a'}Q_h(s',a')$, by employing reverse expectile regression to update the constraint violation value network:
\begin{equation}
\label{cost value}
\begin{aligned}
    \mathcal{L}_{V_h}&=\mathbb{E}_{(s,a)\sim \mathcal{D}\cup \mathcal{D}_{\hat{T}}}[\mathcal{L}_{\text{rev}}^\tau(Q_h(s,a)-V_h(s))],\\
    \mathcal{L}_{Q_h}&=\mathbb{E}_{(s,a,s')\sim \mathcal{D}}[((1-\gamma)h(s)+\gamma\max(h(s),V_h(s'))-Q_h(s,a))^2]\\&+\mathbb{E}_{(s,a)\sim \mathcal{D}_{\hat{T}}}[((1-\gamma)h(s)+\gamma\max(h(s),\max_{s'\in\hat{T}(s,a)}V_h(s'))-Q_h(s,a))^2],
\end{aligned}
\end{equation}
where $\mathcal{L}_{\text{rev}}^\tau(u)=|\tau-\mathbb{I}(u>0)|u^2$, $\tau$ is a hyperparameter.
Finally, suppose that for a given offline safe RL algorithm, the reward critic learning loss is denoted by $\mathcal{L}_{\text{reward}}$, and the policy learning loss by $\mathcal{L}_{\text{policy}}$.
We then directly perform learning using the original offline dataset $\mathcal{D}$
\begin{equation}
\begin{aligned}
    \mathcal{L}_{\text{\MethodName reward}} = \mathbb{E}_{(s,a,s')\sim \mathcal{D}}[\mathcal{L}_{\text{reward}}],\ 
    \mathcal{L}_{\text{\MethodName policy}} = \mathbb{E}_{(s,a)\sim \mathcal{D}}[\mathcal{L}_{\text{policy}}(Q_h,V_h)],
\end{aligned}
\end{equation}
where $\mathcal{L}_{\text{policy}}(Q_h,V_h)$ denotes the policy learning loss that uses the $Q_h$ and $V_h$ obtained in \cref{cost value} as the cost critics, and performs policy optimization based on $\mathcal{L}_{\text{policy}}$.
We derive the following theorem to show that, even without making any assumptions about model error, integrating \MethodName with an arbitrary algorithm does not negatively affect the safety of the resulting policy
\begin{theorem}
\label{th_3}
   For any $\mathcal{D}$ and $\hat{T}$, let $V_{h,\text{orgin}}$ denote the optimal feasible value function learned by a baseline, and let $V_{h,\text{PROCO}}$ denote the optimal feasible value function learned by combining the same baseline with PROCO. Then, for any $s_t\in \mathcal{D}$, we have:
   \begin{equation}
       V_{h,\text{PROCO}}(s_t)\geq V_{h,\text{orgin}}(s_t)\ \text{or}\ V_{h,\text{PROCO}}(s_t)>0.
   \end{equation}
   Thus, incorporating \MethodName into safe policy learning will not degrade safety performance due to underestimation of the safety value, compared to learning without \MethodName.
\end{theorem}

\section{Experiments}
In this section, we present our experimental analysis conducted on 17 Safety-Gymnasium~\citep{ji2023safety} tasks from the modified OSRL~\citep{liu2023datasets} dataset to answer the following questions: 
(1) Can \MethodName outperform other baselines across various tasks (Section~\ref{main_exp})?
(2) Do infeasible states in safe-only datasets compromise the safety of policy learning, and can \MethodName address this challenge (Section~\ref{case_study})? 
(3) What is the impact of different components and hyperparameters on \MethodName’s performance (\cref{sens analysis})?
(4) Can \MethodName be used to datasets with more unsafe samples (\cref{sec_unsafe})?

\begin{table*}[t!]
\renewcommand\arraystretch{1.05}
\caption{Overall normalized rewards and costs.
Each value is averaged over 20 evaluation episodes, and 3 random seeds. \textbf{{\color[HTML]{000000} Bold}}: Safe agents. {\color[HTML]{9B9B9B} Gray}: Unsafe agents. \textbf{{\color[HTML]{3531FF} Blue}}: Safe agent with the highest reward.}
\label{main results}
\resizebox{\textwidth}{!}{
\begin{tabular}{c|cccccccccccc|cccccc}
\hline
                       & \multicolumn{2}{c}{BC}                                      & \multicolumn{2}{c}{CDT}                                    & \multicolumn{2}{c}{CCAC}                                    & \multicolumn{2}{c}{LSPC}                                                    & \multicolumn{2}{c}{CAPS}                                    & \multicolumn{2}{c|}{FISOR}                                  & \multicolumn{2}{c}{LSPC PROCO}                                              & \multicolumn{2}{c}{CAPS PROCO}                                              & \multicolumn{2}{c}{FISOR PROCO}                                              \\
\multirow{-2}{*}{Task} & r↑                           & c↓                           & r↑                          & c↓                           & r↑                           & c↓                           & r↑                                   & c↓                                   & r↑                           & c↓                           & r↑                           & c↓                           & r↑                                   & c↓                                   & r↑                                   & c↓                                   & r↑                                    & c↓                                   \\ \hline
PointButton1           & {\color[HTML]{9B9B9B} 0.13}  & {\color[HTML]{9B9B9B} 5.20}  & {\color[HTML]{9B9B9B} 0.47} & {\color[HTML]{9B9B9B} 13.27} & {\color[HTML]{9B9B9B} 0.00}  & {\color[HTML]{9B9B9B} 6.54}  & {\color[HTML]{9B9B9B} 0.16}          & {\color[HTML]{9B9B9B} 4.37}          & {\color[HTML]{9B9B9B} 0.13}  & {\color[HTML]{9B9B9B} 3.37}  & {\color[HTML]{9B9B9B} 0.49}  & {\color[HTML]{9B9B9B} 11.50} & {\color[HTML]{9B9B9B} 0.09}          & {\color[HTML]{9B9B9B} 1.28}          & {\color[HTML]{9B9B9B} -0.03}         & {\color[HTML]{9B9B9B} 1.13}          & {\color[HTML]{9B9B9B} 0.01}           & {\color[HTML]{9B9B9B} 1.15}          \\
PointButton2           & {\color[HTML]{9B9B9B} 0.28}  & {\color[HTML]{9B9B9B} 6.68}  & {\color[HTML]{9B9B9B} 0.59} & {\color[HTML]{9B9B9B} 15.27} & {\color[HTML]{9B9B9B} 0.61}  & {\color[HTML]{9B9B9B} 17.87} & {\color[HTML]{9B9B9B} 0.32}          & {\color[HTML]{9B9B9B} 10.04}         & {\color[HTML]{9B9B9B} 0.18}  & {\color[HTML]{9B9B9B} 4.90}  & {\color[HTML]{9B9B9B} 0.46}  & {\color[HTML]{9B9B9B} 14.69} & {\color[HTML]{9B9B9B} 0.14}          & {\color[HTML]{9B9B9B} 3.02}          & {\color[HTML]{3531FF} \textbf{0.02}} & {\color[HTML]{3531FF} \textbf{0.63}} & {\color[HTML]{9B9B9B} 0.08}           & {\color[HTML]{9B9B9B} 2.76}          \\
PointGoal1             & {\color[HTML]{9B9B9B} 0.58}  & {\color[HTML]{9B9B9B} 3.30}  & {\color[HTML]{9B9B9B} 0.75} & {\color[HTML]{9B9B9B} 5.01}  & {\color[HTML]{9B9B9B} 0.00}  & {\color[HTML]{9B9B9B} 5.29}  & {\color[HTML]{9B9B9B} 0.71}          & {\color[HTML]{9B9B9B} 3.10}          & {\color[HTML]{9B9B9B} 0.61}  & {\color[HTML]{9B9B9B} 3.49}  & {\color[HTML]{9B9B9B} 0.73}  & {\color[HTML]{9B9B9B} 5.70}  & \textbf{0.21}                        & \textbf{0.23}                        & \textbf{0.02}                        & \textbf{0.00}                        & {\color[HTML]{3531FF} \textbf{0.35}}  & {\color[HTML]{3531FF} \textbf{0.96}} \\
PointGoal2             & {\color[HTML]{9B9B9B} 0.53}  & {\color[HTML]{9B9B9B} 8.07}  & {\color[HTML]{9B9B9B} 0.78} & {\color[HTML]{9B9B9B} 13.90} & {\color[HTML]{9B9B9B} 0.87}  & {\color[HTML]{9B9B9B} 14.15} & {\color[HTML]{9B9B9B} 0.59}          & {\color[HTML]{9B9B9B} 8.55}          & {\color[HTML]{9B9B9B} 0.45}  & {\color[HTML]{9B9B9B} 5.81}  & {\color[HTML]{9B9B9B} 0.59}  & {\color[HTML]{9B9B9B} 17.10} & {\color[HTML]{3531FF} \textbf{0.11}} & {\color[HTML]{3531FF} \textbf{0.45}} & \textbf{0.05}                        & \textbf{0.66}                        & \textbf{0.08}                         & \textbf{0.38}                        \\
PointPush1             & {\color[HTML]{9B9B9B} 0.24}  & {\color[HTML]{9B9B9B} 4.65}  & {\color[HTML]{9B9B9B} 0.30} & {\color[HTML]{9B9B9B} 4.65}  & {\color[HTML]{9B9B9B} 0.02}  & {\color[HTML]{9B9B9B} 1.18}  & {\color[HTML]{9B9B9B} 0.17}          & {\color[HTML]{9B9B9B} 2.21}          & {\color[HTML]{9B9B9B} 0.19}  & {\color[HTML]{9B9B9B} 3.96}  & {\color[HTML]{9B9B9B} 0.34}  & {\color[HTML]{9B9B9B} 3.70}  & {\color[HTML]{9B9B9B} 0.16}          & {\color[HTML]{9B9B9B} 2.09}          & \textbf{0.11}                        & \textbf{0.26}                        & {\color[HTML]{3531FF} \textbf{0.17}}  & {\color[HTML]{3531FF} \textbf{0.86}} \\
PointPush2             & {\color[HTML]{9B9B9B} 0.22}  & {\color[HTML]{9B9B9B} 3.70}  & {\color[HTML]{9B9B9B} 0.28} & {\color[HTML]{9B9B9B} 4.51}  & {\color[HTML]{9B9B9B} -0.08} & {\color[HTML]{9B9B9B} 5.95}  & {\color[HTML]{9B9B9B} 0.13}          & {\color[HTML]{9B9B9B} 5.76}          & {\color[HTML]{9B9B9B} 0.18}  & {\color[HTML]{9B9B9B} 3.58}  & {\color[HTML]{9B9B9B} 0.27}  & {\color[HTML]{9B9B9B} 6.61}  & {\color[HTML]{9B9B9B} 0.15}          & {\color[HTML]{9B9B9B} 1.79}          & {\color[HTML]{9B9B9B} 0.09}          & {\color[HTML]{9B9B9B} 1.61}          & {\color[HTML]{3531FF} \textbf{0.12}}  & {\color[HTML]{3531FF} \textbf{0.30}} \\
CarButton1             & {\color[HTML]{9B9B9B} -0.09} & {\color[HTML]{9B9B9B} 3.32}  & {\color[HTML]{9B9B9B} 0.25} & {\color[HTML]{9B9B9B} 15.53} & {\color[HTML]{9B9B9B} 0.03}  & {\color[HTML]{9B9B9B} 17.26} & {\color[HTML]{9B9B9B} -0.05}         & {\color[HTML]{9B9B9B} 6.51}          & {\color[HTML]{9B9B9B} 0.06}  & {\color[HTML]{9B9B9B} 4.51}  & {\color[HTML]{9B9B9B} 0.47}  & {\color[HTML]{9B9B9B} 28.01} & {\color[HTML]{9B9B9B} -0.03}         & {\color[HTML]{9B9B9B} 1.92}          & {\color[HTML]{9B9B9B} -0.02}         & {\color[HTML]{9B9B9B} 2.02}          & {\color[HTML]{3531FF} \textbf{-0.03}} & {\color[HTML]{3531FF} \textbf{0.60}} \\
CarButton2             & {\color[HTML]{9B9B9B} -0.09} & {\color[HTML]{9B9B9B} 4.64}  & {\color[HTML]{9B9B9B} 0.35} & {\color[HTML]{9B9B9B} 19.98} & {\color[HTML]{9B9B9B} 0.01}  & {\color[HTML]{9B9B9B} 11.94} & {\color[HTML]{9B9B9B} -0.17}         & {\color[HTML]{9B9B9B} 2.49}          & {\color[HTML]{9B9B9B} -0.12} & {\color[HTML]{9B9B9B} 4.02}  & {\color[HTML]{9B9B9B} 0.57}  & {\color[HTML]{9B9B9B} 25.96} & {\color[HTML]{9B9B9B} -0.11}         & {\color[HTML]{9B9B9B} 1.88}          & \textbf{-0.09}                       & \textbf{0.34}                        & {\color[HTML]{3531FF} \textbf{0.00}}  & {\color[HTML]{3531FF} \textbf{0.99}} \\
CarGoal1               & {\color[HTML]{9B9B9B} 0.36}  & {\color[HTML]{9B9B9B} 2.15}  & {\color[HTML]{9B9B9B} 0.71} & {\color[HTML]{9B9B9B} 5.04}  & \textbf{0.00}                & \textbf{0.00}                & {\color[HTML]{9B9B9B} 0.47}          & {\color[HTML]{9B9B9B} 1.58}          & {\color[HTML]{9B9B9B} 0.43}  & {\color[HTML]{9B9B9B} 2.15}  & {\color[HTML]{9B9B9B} 0.75}  & {\color[HTML]{9B9B9B} 4.10}  & {\color[HTML]{3531FF} \textbf{0.20}} & {\color[HTML]{3531FF} \textbf{0.68}} & \textbf{0.14}                        & \textbf{0.14}                        & \textbf{0.13}                         & \textbf{0.02}                        \\
CarGoal2               & {\color[HTML]{9B9B9B} 0.22}  & {\color[HTML]{9B9B9B} 2.96}  & {\color[HTML]{9B9B9B} 0.67} & {\color[HTML]{9B9B9B} 11.43} & {\color[HTML]{9B9B9B} 0.48}  & {\color[HTML]{9B9B9B} 14.73} & {\color[HTML]{9B9B9B} 0.25}          & {\color[HTML]{9B9B9B} 3.85}          & {\color[HTML]{9B9B9B} 0.18}  & {\color[HTML]{9B9B9B} 3.41}  & {\color[HTML]{9B9B9B} 0.80}  & {\color[HTML]{9B9B9B} 14.56} & {\color[HTML]{3531FF} \textbf{0.18}} & {\color[HTML]{3531FF} \textbf{0.52}} & \textbf{0.07}                        & \textbf{0.30}                        & \textbf{0.03}                         & \textbf{0.52}                        \\
CarPush1               & {\color[HTML]{9B9B9B} 0.19}  & {\color[HTML]{9B9B9B} 1.10}  & {\color[HTML]{9B9B9B} 0.32} & {\color[HTML]{9B9B9B} 2.86}  & {\color[HTML]{9B9B9B} -1.51} & {\color[HTML]{9B9B9B} 0.64}  & {\color[HTML]{3531FF} \textbf{0.19}} & {\color[HTML]{3531FF} \textbf{0.91}} & {\color[HTML]{9B9B9B} 0.18}  & {\color[HTML]{9B9B9B} 2.39}  & {\color[HTML]{9B9B9B} 0.42}  & {\color[HTML]{9B9B9B} 1.91}  & \textbf{0.10}                        & \textbf{0.55}                        & {\color[HTML]{000000} \textbf{0.16}} & {\color[HTML]{000000} \textbf{0.20}} & {\color[HTML]{000000} \textbf{0.16}}  & {\color[HTML]{000000} \textbf{0.93}} \\
CarPush2               & {\color[HTML]{9B9B9B} 0.08}  & {\color[HTML]{9B9B9B} 4.14}  & {\color[HTML]{9B9B9B} 0.23} & {\color[HTML]{9B9B9B} 9.88}  & {\color[HTML]{9B9B9B} 0.06}  & {\color[HTML]{9B9B9B} 3.02}  & {\color[HTML]{9B9B9B} 0.09}          & {\color[HTML]{9B9B9B} 3.63}          & {\color[HTML]{9B9B9B} 0.11}  & {\color[HTML]{9B9B9B} 3.67}  & {\color[HTML]{9B9B9B} 0.39}  & {\color[HTML]{9B9B9B} 8.04}  & {\color[HTML]{9B9B9B} 0.12}          & {\color[HTML]{9B9B9B} 1.30}          & {\color[HTML]{9B9B9B} 0.03}          & {\color[HTML]{9B9B9B} 1.50}          & {\color[HTML]{3531FF} \textbf{0.01}}  & {\color[HTML]{3531FF} \textbf{0.04}} \\
SwimmerVelocityV1      & {\color[HTML]{9B9B9B} 0.46}  & {\color[HTML]{9B9B9B} 5.70}  & {\color[HTML]{9B9B9B} 0.70} & {\color[HTML]{9B9B9B} 4.59}  & {\color[HTML]{9B9B9B} 0.22}  & {\color[HTML]{9B9B9B} 9.40}  & {\color[HTML]{9B9B9B} 0.62}          & {\color[HTML]{9B9B9B} 6.75}          & {\color[HTML]{9B9B9B} 0.51}  & {\color[HTML]{9B9B9B} 7.98}  & {\color[HTML]{9B9B9B} -0.05} & {\color[HTML]{9B9B9B} 4.51}  & {\color[HTML]{3531FF} \textbf{0.05}} & {\color[HTML]{3531FF} \textbf{0.58}} & \textbf{0.03}                        & \textbf{0.08}                        & \textbf{0.02}                         & \textbf{0.12}                        \\
HopperVelocityV1       & {\color[HTML]{9B9B9B} 0.35}  & {\color[HTML]{9B9B9B} 1.33}  & {\color[HTML]{9B9B9B} 0.63} & {\color[HTML]{9B9B9B} 3.44}  & {\color[HTML]{9B9B9B} 0.08}  & {\color[HTML]{9B9B9B} 7.00}  & {\color[HTML]{9B9B9B} 0.08}          & {\color[HTML]{9B9B9B} 2.19}          & {\color[HTML]{9B9B9B} 0.66}  & {\color[HTML]{9B9B9B} 8.22}  & {\color[HTML]{9B9B9B} 0.07}  & {\color[HTML]{9B9B9B} 3.23}  & {\color[HTML]{9B9B9B} 0.17}          & {\color[HTML]{9B9B9B} 5.86}          & {\color[HTML]{3531FF} \textbf{0.26}} & {\color[HTML]{3531FF} \textbf{0.06}} & \textbf{0.11}                         & \textbf{0.07}                        \\
HalfCheetahVelocityV1  & {\color[HTML]{9B9B9B} 0.95}  & {\color[HTML]{9B9B9B} 4.27}  & {\color[HTML]{9B9B9B} 0.95} & {\color[HTML]{9B9B9B} 1.42}  & {\color[HTML]{9B9B9B} 1.51}  & {\color[HTML]{9B9B9B} 97.03} & {\color[HTML]{9B9B9B} 1.03}          & {\color[HTML]{9B9B9B} 45.52}         & {\color[HTML]{9B9B9B} 0.98}  & {\color[HTML]{9B9B9B} 9.17}  & {\color[HTML]{9B9B9B} 1.03}  & {\color[HTML]{9B9B9B} 16.13} & \textbf{0.41}                        & \textbf{0.00}                        & \textbf{0.00}                        & \textbf{0.00}                        & {\color[HTML]{3531FF} \textbf{0.48}}  & {\color[HTML]{3531FF} \textbf{0.00}} \\
Walker2dVelocityV1     & {\color[HTML]{9B9B9B} 0.69}  & {\color[HTML]{9B9B9B} 2.67}  & \textbf{0.79}               & \textbf{0.55}                & {\color[HTML]{9B9B9B} 0.05}  & {\color[HTML]{9B9B9B} 1.42}  & {\color[HTML]{3531FF} \textbf{0.80}} & {\color[HTML]{3531FF} \textbf{0.00}} & \textbf{0.78}                & \textbf{0.39}                & {\color[HTML]{9B9B9B} 0.15}  & {\color[HTML]{9B9B9B} 5.71}  & {\color[HTML]{9B9B9B} 0.50}          & {\color[HTML]{9B9B9B} 2.08}          & {\color[HTML]{9B9B9B} 0.53}          & {\color[HTML]{9B9B9B} 2.83}          & \textbf{0.09}                         & \textbf{0.58}                        \\
AntVelocityV1          & {\color[HTML]{9B9B9B} 0.98}  & {\color[HTML]{9B9B9B} 11.54} & {\color[HTML]{9B9B9B} 0.99} & {\color[HTML]{9B9B9B} 1.85}  & \textbf{-1.01}               & \textbf{0.00}                & {\color[HTML]{9B9B9B} 0.99}          & {\color[HTML]{9B9B9B} 11.44}         & {\color[HTML]{9B9B9B} 0.98}  & {\color[HTML]{9B9B9B} 13.25} & {\color[HTML]{9B9B9B} 1.01}  & {\color[HTML]{9B9B9B} 10.47} & {\color[HTML]{3531FF} \textbf{0.81}} & {\color[HTML]{3531FF} \textbf{0.16}} & \textbf{0.50}                        & \textbf{0.00}                        & \textbf{0.52}                         & \textbf{0.00}                        \\ \hline
Average                & {\color[HTML]{9B9B9B} 0.36}  & {\color[HTML]{9B9B9B} 4.44}  & {\color[HTML]{9B9B9B} 0.57} & {\color[HTML]{9B9B9B} 7.84}  & {\color[HTML]{9B9B9B} 0.08}  & {\color[HTML]{9B9B9B} 12.55} & {\color[HTML]{9B9B9B} 0.38}          & {\color[HTML]{9B9B9B} 6.99}          & {\color[HTML]{9B9B9B} 0.38}  & {\color[HTML]{9B9B9B} 4.96}  & {\color[HTML]{9B9B9B} 0.50}  & {\color[HTML]{9B9B9B} 10.70} & {\color[HTML]{9B9B9B} 0.19}          & {\color[HTML]{9B9B9B} 1.44}          & \textbf{0.11}                        & \textbf{0.69}                        & {\color[HTML]{3531FF} \textbf{0.14}}  & {\color[HTML]{3531FF} \textbf{0.60}} \\ \hline
Safe Percent           & \multicolumn{2}{c}{0/17}                                    & \multicolumn{2}{c}{1/17}                                   & \multicolumn{2}{c}{3/17}                                    & \multicolumn{2}{c}{2/17}                                                    & \multicolumn{2}{c}{1/17}                                    & \multicolumn{2}{c|}{0/17}                                   & \multicolumn{2}{c}{8/17}                                                    & \multicolumn{2}{c}{12/17}                                                   & \multicolumn{2}{c}{15/17}                                                    \\ \hline
Safety Improvement     & \multicolumn{2}{c}{/}                                       & \multicolumn{2}{c}{/}                                      & \multicolumn{2}{c}{/}                                       & \multicolumn{2}{c}{/}                                                       & \multicolumn{2}{c}{/}                                       & \multicolumn{2}{c|}{/}                                      & \multicolumn{2}{c}{555\%}                                                   & \multicolumn{2}{c}{427\%}                                                   & \multicolumn{2}{c}{1011\%}                                                   \\ \hline
\end{tabular}
}
\end{table*}

\subsection{Baselines and Tasks}
To evaluate \MethodName and its flexibility in integrating with other algorithms, we select three advanced offline safe RL methods. \textbf{FISOR}~\cite{zhengsafe} and \textbf{LSPC}~\cite{koiralalatent} are hard-constraint algorithms, while \textbf{CAPS}(IQL)~\cite{chemingui2025constraint}, originally a soft-constraint method, shares a similar policy training framework and can be extended to hard-constraint settings. Combining each with \MethodName yields three variants: \textbf{FISOR \MethodName}, \textbf{LSPC \MethodName}, and \textbf{CAPS \MethodName}.
To compare with BC–based methods on safe-only datasets, we include \textbf{BC} and \textbf{CDT}~\citep{liu2023constrained} as baselines. Finally, we also include \textbf{CCAC}~\cite{guo2025constraint}, an advanced soft-constraint algorithm based on conservative value estimation, as an additional baseline. All soft-constraint baselines adopt 0 as the cost limit. Reward returns are normalized using the offline dataset’s minimum and maximum, while costs are scaled by 10. Outcomes with normalized cost $\leq 1$ are considered safe, with higher rewards indicating better performance among safe algorithms.

Our experiments use 12 navigation and 5 velocity tasks from Safety-Gymnasium. The navigation tasks involve two robot types (Point and Car) across three scenarios—Button, Goal, and Push—with two tasks per scenario. The velocity tasks include five robots: Ant, HalfCheetah, Hopper, Swimmer, and Walker2d.
To evaluate \MethodName, we use the most challenging OSRL dataset, where task trajectories exhibit varying constraint violations. We construct safe-only datasets by removing all unsafe data while preserving diverse feasibility levels. Additionally, 100 unsafe samples per task are retained for validating LLM-generated cost functions. Further details are in Appendix~\ref{detail task base}.

\begin{figure*}[t]
    \centering
    \includegraphics[width=0.99\textwidth]{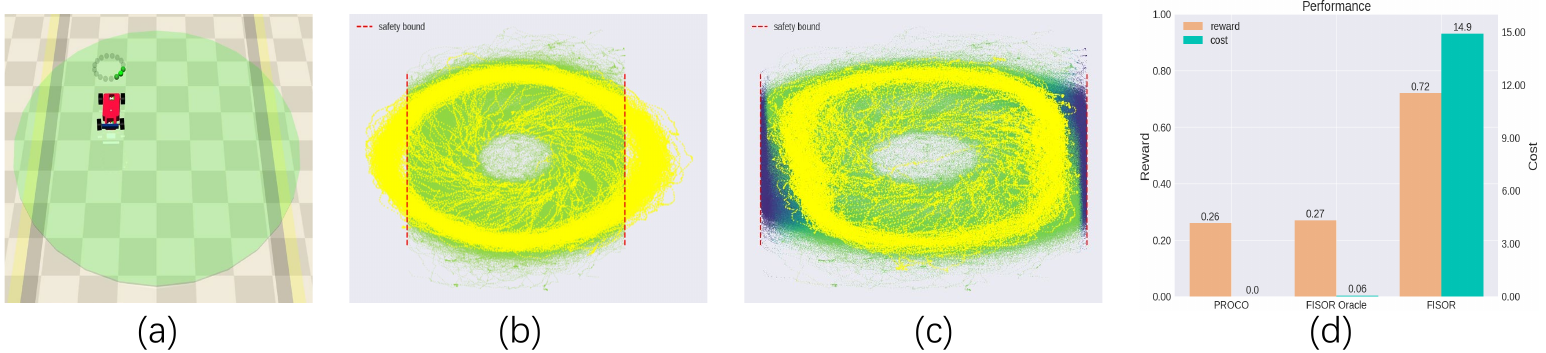}
    \caption{(a) Visualization of the Circle task. (b) and (c) Visualization of FISOR and \MethodName performance, where the red dashed line denotes the safety bound, while the yellow dots indicate the state visitation distribution of the policy. The remaining points represent the value estimations of the constraint violation value function for samples in $\mathcal{D}$, where darker colors correspond to higher estimated violation values. (d) Final performance comparison of FISOR and \MethodName.}
    \label{fig:vis}
\end{figure*}

\begin{figure*}[t]
    \centering
    \includegraphics[width=0.99\textwidth]{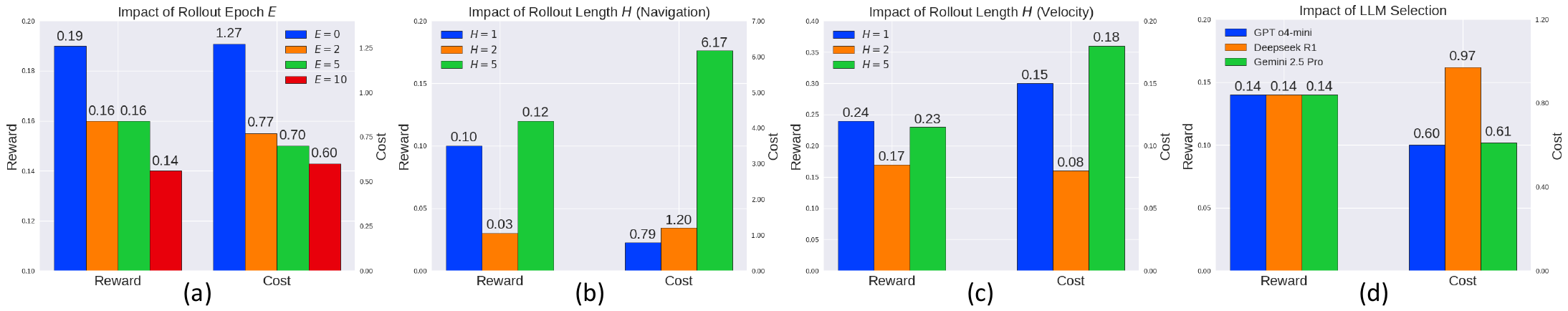}
    \caption{(a) Analysis on rollout epoch $E$. (b), (c) Analysis on rollout length $H$. (d) Analysis on LLM selection.}
    \label{fig:sens}
\end{figure*}

\subsection{Competitive Results}
\label{main_exp}

We first present the overall performance of \MethodName variants and all baselines across Safety-Gymnasium tasks, summarized in~\cref{main results}.
The soft-constraint baseline CCAC exhibits the weakest safety, showing that conventional TD learning with conservative estimation fails to adequately constrain the policy, causing distributional shift and severe violations. BC-based methods improve safety by restricting policies to safe-only data but still cannot reliably distinguish feasible from infeasible states, resulting in notable violations. Existing SOTA offline safe RL algorithms—LSPC, CAPS, and FISOR—perform well with many unsafe samples but fail on safe-only datasets, focusing on reward maximization and incurring significant safety violations.
In contrast, integrating these algorithms with \MethodName enables effective feasibility identification via model-based rollouts and conservative cost functions, yielding substantial safety gains. Safety performance across all three methods increases by over 400\%, with a maximum improvement exceeding 1000\%, demonstrating \MethodName’s effectiveness on safe-only datasets. Due to its superior performance, we adopt FISOR \MethodName as the default variant, referred to simply as \MethodName in subsequent experiments.

\subsection{Case Study}
\label{case_study}
Here, we aim to verify whether infeasible states in a safe-only dataset can adversely affect safe policy learning. To this end, we conduct a simple visualization experiment on the Ant Circle task. As illustrated in Figure~\ref{fig:vis}(a), the Circle task requires the agent to move along the circumference of a circle as closely as possible, while crossing the left or right boundaries is considered unsafe.

\begin{wrapfigure}{r}{0.5\textwidth} 
    \centering
    \includegraphics[width=0.48\textwidth]{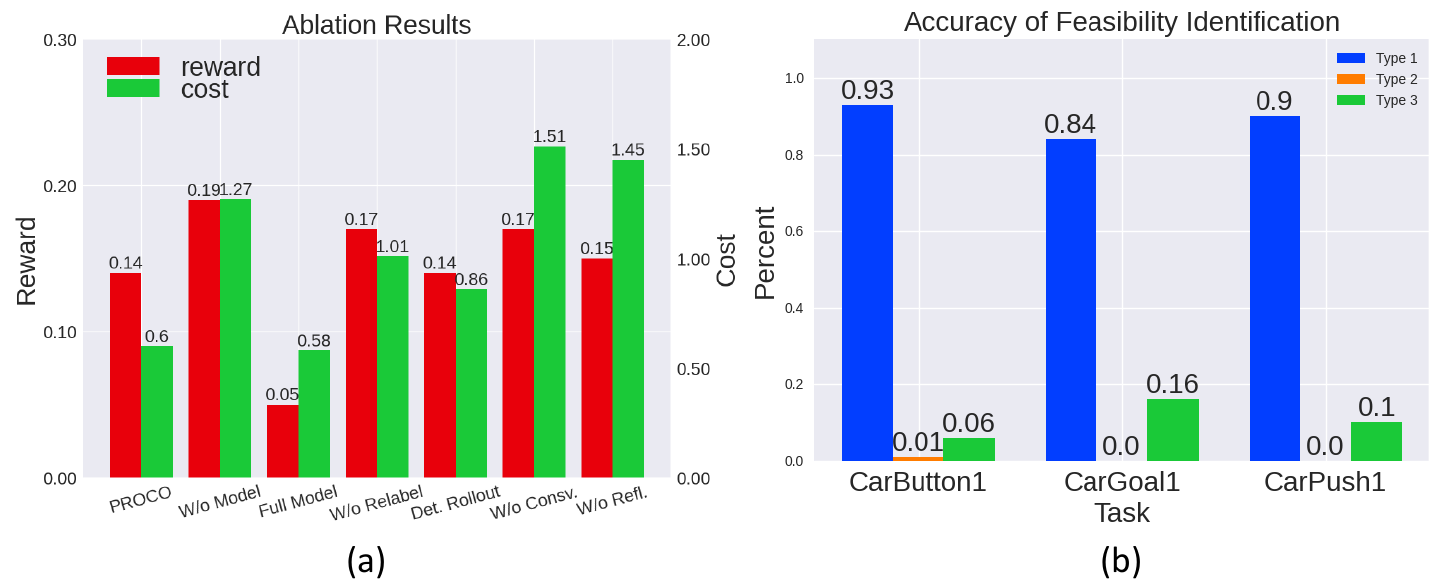}
    \caption{(a) Ablation study results. (b) Feasibility identification accuracy analysis.}
   \label{fig:abl}
\end{wrapfigure}


We compare FISOR with \MethodName. As shown in Figure~\ref{fig:vis}(b), without unsafe samples, FISOR’s constraint violation value function cannot distinguish feasible from infeasible offline samples. Consequently, the learned policy prioritizes reward maximization, and even with BC-based extraction, it deviates from the offline distribution, generating many out-of-distribution unsafe samples.
In contrast, as shown in Figure~\ref{fig:vis}(c), \MethodName leverages the conservative cost function together with the dynamics model to successfully identify samples near the safety boundary as infeasible. This enables the policy to adjust its behavior in time, preventing it from drifting beyond the safe region of the offline dataset and thereby avoiding constraint violations.
Finally, as shown in Figure~\ref{fig:vis}(d), \MethodName successfully achieves zero safety violations, attaining performance comparable to \textbf{FISOR Oracle} (which is trained with access to a large amount of additional unsafe samples), whereas FISOR produces a substantial number of unsafe behaviors.

\subsection{Ablation and Analysis Studies}
\label{sens analysis}


In this section, we first conduct ablation studies to evaluate the contributions of key components in \MethodName.
The following variants are considered: 
(1) \textbf{W/o Model} omits the use of a dynamics model for generating unsafe samples, instead solely applying the conservative cost function to relabel the offline dataset.
(2) \textbf{Full Model} utilizes all model rollout data to update all the value functions and the policy.
(3) \textbf{W/o Relabel} applies the conservative cost function only to label the model rollout data and does not relabel the offline dataset.
(4) \textbf{Det. Rollout} executes model rollouts deterministically with the policy, omitting the injection of extra noise.
(5) \textbf{W/o Consv.} does not require the LLM to generate a cost function that is more conservative than the given safety constraint description.
(6) \textbf{W/o Refl.} does not use the additional unsafe samples $\mathcal{D}_{\text{unsafe}}$ or the safe offline dataset to validate and provide feedback on the correctness and conservativeness of the LLM-generated cost function.
As shown in~\cref{fig:abl}(a), when the dynamics model is not used, safety performance drops significantly, confirming the effectiveness of model-based feasibility identification. On the other hand, employing all model rollout data for updates to all value functions and the policy preserves high safety performance but significantly reduces reward performance, suggesting that errors accumulated during model rollouts can destabilize value learning and negatively impact reward estimation.
Meanwhile, the decline in safety performance for W/o Relabel and Det. Rollout confirms the effectiveness of relabeling offline data with the conservative cost function and adding extra noise during model rollouts to enhance data diversity. Finally, the poor safety performance of W/o Consv. and W/o Refl. highlights the importance of generating a more conservative cost function for effective feasibility identification, as well as the necessity of validating and refining the LLM-generated cost function using existing data.
More detailed results are provided in Appendix~\ref{more results}.

Next, we evaluate the accuracy of the feasibility identification learned by \MethodName. Specifically, we compare the feasible values predicted by \MethodName with those produced by the FISOR Oracle, and categorize the results into three types: \textbf{Type 1}, where both methods yield the same sign; \textbf{Type 2}, where \MethodName classifies a state as safe while the FISOR Oracle classifies it as unsafe; and \textbf{Type 3}, where \MethodName classifies a state as unsafe while the FISOR Oracle classifies it as safe. The results are shown in \cref{fig:abl}(b).
We observe that \MethodName achieves a high level of accuracy in feasibility identification, producing predictions that are comparable to those of the FISOR Oracle, even in the absence of unsafe samples. Moreover, Type 2 cases are almost nonexistent, indicating that \MethodName exhibits extremely high precision in identifying infeasible samples, further validating the effectiveness of the proposed model-based framework.

Then, we investigate the impact of different hyperparameters on \MethodName’s performance. During rollouts, three hyperparameters play a crucial role: rollout batch size $b$, the number of rollout epochs $E$, and rollout length $H$. Specifically, $b$ denotes the number of samples drawn, $H$ is the number of steps rolled out per sample, and $E$ indicates how many times the previous process is repeated per rollout. Since both $b$ and $E$ determine the total amount of rollout data, we focus our analysis on $E$ and $H$.
First, as shown in~\cref{fig:sens}(a), increasing $E$ leads to a noticeable decline in reward and cost. This indicates that a larger amount of rollout data enables more accurate cost value learning, further confirming the effectiveness of model-based feasibility identification.
Second, \cref{fig:sens}(b) and (c) show that in Velocity tasks, where the dynamics model is relatively accurate, increasing $H$ does not noticeably affect policy performance. In contrast, in Navigation tasks, where model learning is more challenging, larger values of $H$ lead to a certain degradation in safety performance. These results confirm the necessity of reducing $H$ while employing a conservative cost function when an accurate dynamics model is difficult to learn.
In the end, \cref{fig:sens}(d) presents the effect of different LLMs on performance. GPT o4-mini and Gemini 2.5 Pro demonstrate comparable overall results, whereas Deepseek R1 shows lower safety performance. This suggests that the choice of LLM can impact \MethodName’s effectiveness, with more capable LLMs generally yielding better outcomes.

\begin{figure*}[t]
    \centering
    \includegraphics[width=0.99\textwidth]{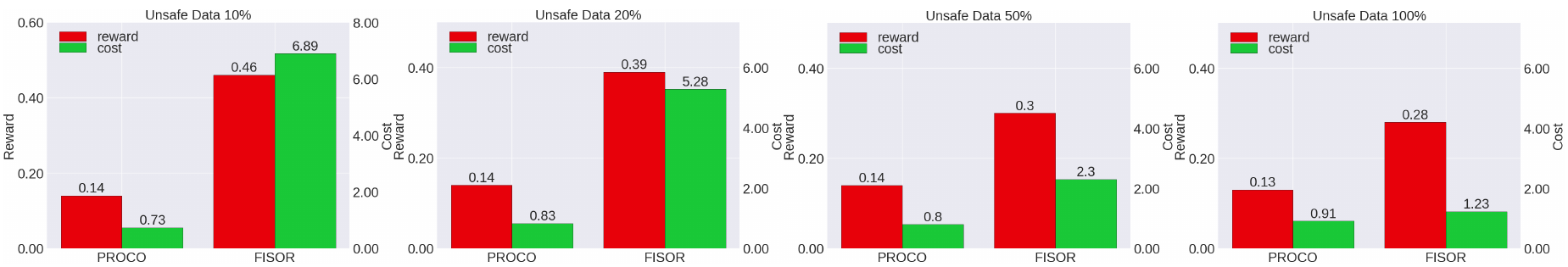}
    \caption{Performance analysis under different unsafe data percent.}
    \label{fig:unsafe_percent}
\end{figure*}

\subsection{Performance with Different Unsafe Data Number}
\label{sec_unsafe}
Finally, we evaluate \MethodName and FISOR under varying amounts of unsafe data (10\%, 20\%, 50\%, 100\% of the OSRL unsafe dataset) across 17 tasks. The results are shown in \cref{fig:unsafe_percent}.
We observe that as the amount of unsafe data increases, the reward performance of \MethodName remains largely stable, while the cost exhibits a moderate increase. This is likely because incorporating more unsafe samples into the BC term introduces a stronger bias toward unsafe actions. Meanwhile, these results suggest that model errors have a limited impact when the rollout length is small.
Importantly, when unsafe data are scarce, \MethodName consistently achieves substantially better safety performance than FISOR. Only when the unsafe data reach 100\% does the use of \MethodName become a trade-off: incorporating \MethodName sacrifices some reward performance in exchange for improved safety, whereas omitting it prioritizes reward optimization.

\section{Final Remarks}
In this work, we propose \MethodName, a novel algorithm for learning safe policies offline when unsafe samples are scarce or entirely absent.
\MethodName first leverages the offline dataset to learn a dynamics model, while simultaneously employing an LLM to generate a conservative cost function tailored to the task. Based on the learned dynamics model and cost function, it then performs branched rollouts from the offline data samples to simulate their potential future evolutions. The generated rollout data are further incorporated into HJ reachability analysis for feasibility identification, which in turn guides policy extraction to ensure both safety and effectiveness.
Extensive experiments across diverse tasks demonstrate the superior safety performance of \MethodName.
In the future, leveraging more powerful Vision-Language Models (VLMs) to extend \MethodName to visual tasks and embodied intelligence in a more cost-efficient manner represents a highly promising research direction.


\bibliographystyle{plainnat}
\bibliography{reference}

\appendix

\section{Mathematical Proofs}
\subsection{Proof of Proposition~\ref{pro_1}}
\begin{proof}
    First, we aim to proof that 
    \begin{equation}
        |\max(a_1,b_1)-\max(a_2,b_2)|\leq\max(|a_1-a_2|,|b_1-b_2|).
    \end{equation}
    To start, if $a_1\geq b_1$, then 
    \begin{equation}
        \max(a_1,b_1)-\max(a_2,b_2)=a_1-\max(a_2,b_2)\leq a_1-a_2.
    \end{equation}
    Otherwise, if $b_1\geq a_1$, then 
    \begin{equation}
        \max(a_1,b_1)-\max(a_2,b_2)=b_1-\max(a_2,b_2)\leq b_1-b_2.
    \end{equation}
    Therefore, we have 
    \begin{equation}
        \max(a_1,b_1)-\max(a_2,b_2)\leq \max(a_1-a_2,b_1-b_2).
    \end{equation}
    Similarly, by switching the subscript 1 and 2, we have
    \begin{equation}
        \max(a_2,b_2)-\max(a_1,b_1)\leq \max(a_2-a_1,b_2-b_1),
    \end{equation}
    which is equivalent to
    \begin{equation}
        -(\max(a_1,b_1)-\max(a_2,b_2))\leq\max(-(a_1-a_2),-(b_1-b_2)).
    \end{equation}
    By observing that $\max(x,y)\leq\max(|x|,|y|)$, we have 
    \begin{equation}
        -\max(|a_1-a_2|,|b_1-b_2|)\leq\max(a_1,b_1)-\max(a_2,b_2)\leq\max(|a_1-a_2|,|b_1-b_2|),
    \end{equation}
    which means
    \begin{equation}
    \label{ineq 1}
        |\max(a_1,b_1)-\max(a_2,b_2)|\leq\max(|a_1-a_2|,|b_1-b_2|).
    \end{equation}
    Therefore, we have 
    \begin{align}||\bar{\mathcal{B}^*}Q-\bar{\mathcal{B}^*}Q'||_\infty &= \sup_{s,a}|\bar{\mathcal{B}^*} Q(s,a)-\bar{\mathcal{B}^*}Q'(s,a)|\\&=\gamma\sup_{s,a}|\max\{h(s),\max_{s'\in\hat{T}(s,a)}\min_{a'}Q(s',a')\}\nonumber\\&-\max\{h(s),\max_{s'\in\hat{T}(s,a)}\min_{a'}Q'(s',a')\}| \\\label{ineq 2}&\leq\gamma\sup_{s,a}|\max_{s'\in\hat{T}(s,a)}\min_{a'}Q(s',a')-\max_{s'\in\hat{T}(s,a)}\min_{a'}Q'(s',a')| \\&\leq \gamma\sup_{s,a}\max_{s'\in \hat{T}(s,a)}\max_{a'}|Q(s',a')-Q'(s',a')| \\&\leq\gamma\sup_{s',a'}|Q(s',a')-Q'(s',a')|\\&=\gamma||Q-Q'||_{\infty}, \end{align}
    where inequality~\ref{ineq 2} owns to inequality~\ref{ineq 1}. Thus, $\bar{\mathcal{B}^*}$ is a $\gamma$ contraction mapping in the $\infty$-norm.
    Then, for any $Q, Q'$, if $Q\geq Q'$ pointwise, then 
    \begin{equation}
        \begin{aligned}\bar{\mathcal{B}^*}Q&=(1-\gamma)h(s)+\gamma\max\{h(s),\max_{s'\in\hat{T}(s,a)}\min_{a'}Q(s',a')\}\\&\geq(1-\gamma)h(s)+\gamma\max\{h(s),\min_{a'}Q(s',a')\},s'\sim T(s,a)\\&\geq(1-\gamma)h(s)+\gamma\max\{h(s),\min_{a'}Q'(s',a')\},s'\sim T(s,a)\\&=\mathcal{B}^* Q'.\end{aligned}
    \end{equation}
    Now, let $Q^0$ be any initial $Q$ function, and define $\bar{Q^k}=(\bar{\mathcal{B}^*})^k Q_0,Q^k=(\mathcal{B}^*)^k Q_0$. Then, we have $\bar{Q^k}\geq Q^k,\forall k\in\mathbb{N}$ . Therefore, taking the limits $\bar{Q_h^*}=\lim_{k\to\infty}\bar{Q^k}$ and $Q^*_{h,T}=\lim_{k\to\infty}Q^k$, we obtain $\bar{Q^*_h}(s,a)\geq Q^*_{h,T}(s,a)$ pointwise.
\end{proof}

\subsection{Proof of \cref{th_1}}
\begin{lemma}
\label{lm 1}
    Suppose that Assumption~\ref{ass_1} holds, then $\forall s\in \bar{S_f^*}=S-S_f^*$, it holds that $\forall a,Q^*_h(s,a)\geq h_{\text{min}}+ \gamma^{H^*}(h_{\text{max}}-h_{\text{min}})$, where $Q^*_h$ is the learned action-value function by the feasible Bellman operator, and $h_{\text{min}}$ is the minimum value of $h(s)$, and $h_{\text{max}}$ is the value of $h(s)$ when it is unsafe, that is whenever $h(s)>0$, $h(s)=h_{\text{max}}$.
\end{lemma}
\begin{proof}
    For any trajectory $\tau$ starting from the infeasible state $s_0$, suppose it first violates at horizon $t$, that is for $(s_0,a_0,s_1,a_1,\dots,s_t)$, $h(s_t)=h_{\text{max}}$. Therefore, we have
    \begin{equation}
    \begin{aligned}
        Q^*_h(s_t,a_t)&=(1-\gamma)h(s_t)+\gamma\max\{h(s_t),\min_{a'}Q^*_h(s_{t+1},a_{t+1})\}\\&\geq (1-\gamma)h_{\text{max}}+\gamma h_{\text{max}}=h_{\text{max}}.
    \end{aligned}
    \end{equation}
    Similarly, by recursion we have 
    \begin{equation}
    \begin{aligned}
       Q^*_h(s_{t-k},a_{t-k})&\geq (1-\gamma)h_{\text{min}}(1+\gamma+\dots+\gamma^{k-1})+\gamma^k h_{\text{max}}\\&=h_{\text{min}}+\gamma^k(h_{\text{max}}-h_{\text{min}}).
    \end{aligned}
    \end{equation}
    For that $\gamma<1$ and $h_{\text{max}}-h_{\text{min}}>0$, and the arbitrariness of trajectory $\tau$, we get 
    \begin{equation}
        \forall a_0,Q^*_h(s_0,a_0)\geq h_{\text{min}}+\gamma^t(h_{\text{max}}-h_{\text{min}})\geq h_{\text{min}}+\gamma^{H^*}(h_{\text{max}}-h_{\text{min}}).
    \end{equation}
\end{proof}
Based on Lemma~\ref{lm 1}, we can now proceed to the proof of \cref{th_1}.
\begin{proof}
    For any infeasible state $s_0$, suppose $|\hat{T}|=N$, that is the ensemble dynamics model set has $N$ dynamics models, using the following rollout method
    \begin{itemize}
        \item Using each dynamics model in $\hat{T}$ to obtain the fist next state set $\{s\}_1$ with $N$ states.
        \item For each state in $\{s\}_1$, using each dynamics model to generate next rollout state set.
        \item Repeat the above two step, untill reaching horizon $H^*$.
    \end{itemize}
    This will generate $N^{H^*}$ rollout branches, and due to Assumption~\ref{ass_1} and Assumption~\ref{ass_2}, there will be at least one branch that is unsafe. Therefore, if there is only one unsafe branch, then this will be the ground truth branch, and according to Lemma~\ref{lm 1}, it can be ensured that $\forall a,Q^*_h(s_0,a)\geq h_{\text{min}}+ \gamma^{H^*}(h_{\text{max}}-h_{\text{min}})$. Otherwise, if there are more than one unsafe branch, then we don’t need to figure out the ground truth branch, choosing any one of them to update the action value function can also lead to $\forall a,Q^*_h(s_0,a)\geq h_{\text{min}}+ \gamma^{H^*}(h_{\text{max}}-h_{\text{min}})$.
    Then, according to Proposition~\ref{pro_1} and selecting $\gamma \in (\sqrt[H^*]{\frac{h_{\text{min}}}{h_{\text{min}}-h_{\text{max}}}},1)$, we have 
    \begin{equation}
        \forall a,\bar{Q_h^*}(s_0,a)\geq Q_h^*(s_0,a)\geq h_{\text{min}}+ \gamma^{H^*}(h_{\text{max}}-h_{\text{min}})>0.
    \end{equation}
\end{proof}

\subsection{Proof of \cref{th_2}}
\begin{lemma}{\citep{janner2019trust}}
\label{lm 2}
    Suppose the expected KL-divergence between two transition distributions is bounded as $\max_t\mathbb{E}_{s\sim p_1^t(s)}D_{\text{KL}}(p_1(s'|s)||p_2(s'|s))\leq \delta$, and the initial stat distributions are the same $p_1^0(s)=p_2^0(s)$. Then the distance in the state marginal is bounded as 
    \begin{equation}
        D_{\text{TV}}(p_1^t(s)||p_2^t(s))\leq t\delta.
    \end{equation}
\end{lemma}
\begin{proof}
    Please refer to \cite{janner2019trust}.
\end{proof}
Based on Lemma~\ref{lm 2}, we can now proceed to the proof of \cref{th_2}.
\begin{proof}
    First, according to Assumption~\ref{ass_3}, Assumption~\ref{ass_4} and Lemma~\ref{lm 2}, we have $|\mathbb{E}_{p_1^t}[h(s)]-\mathbb{E}_{p_2^t}[h(s)]|\leq tK\delta$, which means $\mathbb{E}_{p_2^t}[h(s)]\geq \mathbb{E}_{p_1^t}[h(s)]-tK\delta$ holds for any $t$.
    Meanwhile, similar to the proof of Lemma 1, it is clear that $Q_h^{p_2}(s,a)\geq h_{\text{min}}+\gamma^{t}(\mathbb{E}_{p_2^t}[h(s)]-h_{\text{min}})$, where $Q_h^{p_2}$ refers to the learned value function with the feasible Bellman operator and state transition distribution $p_2$. (A critical observation is $\forall a,\mathbb{E}_{s\sim p_2^t}[Q(s,a)]\geq \mathbb{E}_{s\sim p_2^t}[h(s)]$.)

    Therefore, combining the results before and we can get $Q_h^{p_2}(s,a)\geq(1-\gamma^t)h_{\text{min}}+\gamma^t\mathbb{E}_{p_1^t}[h(s)]-\gamma^t tK\delta$ holds for any $t$, which means
    \begin{equation}
        Q_h^{p_2}(s,a)\geq\max_{t\in\{1,\dots,H^*\}}[(1-\gamma^t)h_{\text{min}}+\gamma^t\mathbb{E}_{p_1^t}[h(s)]-\gamma^t tK\delta].
    \end{equation}
    Finally, apply Proposition~\ref{pro_1}, we have
    \begin{equation}
        \begin{aligned}
            \bar{Q_h^*}(s,a)\geq Q_h^{p_2}(s,a)&\geq\max_{t\in\{1,\dots,H^*\}}[(1-\gamma^t)h_{\text{min}}+\gamma^t\mathbb{E}_{p_1^t}[h(s)]-\gamma^t tK\delta]\\&\geq(1-\gamma^{t^*})h_{\text{min}}+\gamma^{t^*}\mathbb{E}_{p_1^{t^*}}[h(s)]-\gamma^{t^*} {t^*}K\delta,
        \end{aligned}
    \end{equation}
    where $t^*=\arg\max_{t\in\{1,\dots,H^*\}}[\mathbb{E}_{p_1^t}[h(s)]]$.
\end{proof}

\subsection{Proof of \cref{th_3}}
\begin{lemma}
\label{lm3}
    Let $x=\{x_1,x_2,\dots,x_n\}$, and let $y_{\text{old}}$ denote the $\tau$-quantile of $x$. Suppose we augment $x$ by adding an arbitrary value $\alpha>0$, and let the resulting $\tau$-quantile be denoted by $y_{\text{new}}$. Then, if $y_{\text{old}}>0$, it necessarily follows that $y_{\text{new}}>0$.
\end{lemma}
\begin{proof}
    Let us define the empirical cumulative distribution function (ECDF). For a given sample set $x=\{x_1,x_2,\dots,x_n\}$, the empirical distribution function $F_n$ is defined as
    \begin{equation}
        F_n(t) = \frac{1}{n} \sum_{i=1}^{n} \mathbf{1}_{\{x_i \leq t\}},
    \end{equation}
    where $\mathbf{1}_{\{x_i \leq t\}}$ is the indicator function that equals 1 if $x_i \leq t$ and 0 otherwise.

    Since $y_{\text{old}}>0$, then we have
    \begin{equation}
        F_n(0)<\tau.
    \end{equation}
    After adding a positive value $\alpha>0$ to $x$, the number of sample points in $(-\infty,0]$ remains unchanged. Let $m$ denote the original number of samples satisfying $x_i\leq 0$, then
    \begin{equation}
        F_n(0)=\frac{m}{n}<\tau.
    \end{equation}
    In the new sample set, we have
    \begin{equation}
        F_{n+1}(0) = \frac{m}{n+1}\leq \frac{m}{n}<\tau.
    \end{equation}
    Therefore, we have
    \begin{equation}
        y_{\text{new}}=\inf\{t:F_{n+1}(t)\geq \tau\}>0.
    \end{equation}
\end{proof}

\begin{lemma}
\label{lm4}
    Let $x=\{x_1,x_2,\dots,x_n\}$, and let $y_{\text{old}}$ denote the $\tau$-quantile of $x$. 
    Suppose we change any $x_t>0$ to another positive value $x'_t>0$, and let the resulting $\tau$-quantile be denoted by $y_{\text{new}}$. Then, if $y_{\text{old}}>0$, it necessarily follows that $y_{\text{new}}>0$.
\end{lemma}
\begin{proof}
    Similarly like above, since $y_{\text{old}}>0$, then we have
    \begin{equation}
        F_n(0)<\tau.
    \end{equation}
    After changing a positive value $x_t>0$ to $x'_t>0$, the number of sample points in $(-\infty,0]$ remains unchanged. Let $m$ denote the original number of samples satisfying $x_i\leq 0$, then
    \begin{equation}
        F_n(0)=\frac{m}{n}<\tau.
    \end{equation}
    In the new sample set, we have
    \begin{equation}
        F'_{n}(0) = \frac{m}{n}=F_n(0)<\tau.
    \end{equation}
    Therefore, we have
    \begin{equation}
        y_{\text{new}}=\inf\{t:F'_{n}(t)\geq \tau\}>0.
    \end{equation}
\end{proof}
\begin{proof}
Let's begin the proof of \cref{th_3}.
Without loss of generality, let $V_{r,\text{orgin}}, V_{h,\text{orgin}}$ denote the reward and feasible value functions obtained without \MethodName, respectively, and let $V_{r,\text{\MethodName}}, V_{h,\text{\MethodName}}$ denote the corresponding value functions obtained with \MethodName.
First, the reward value update rule remains unchanged after incorporating \MethodName. Therefore, $\forall s_t$, we have
\begin{equation}
    V_{r,\text{\MethodName}}(s_t)=V_{r,\text{orgin}}(s_t).
\end{equation}
Next, since $V_h$ is obtained via expectile regression, it effectively fits a quantile of the critic function. For any $s_t$ in $\mathcal{D}$, if none of the future states along its trajectory produces an unsafe branched rollout, then $V_h(s_t)$ is not affected by \MethodName, $V_{h,\text{\MethodName}}(s_t)=V_{h,\text{orgin}}(s_t)$.
Otherwise, suppose that at some future state $s_{t+k}$, an unsafe branched rollout occurs, whose first action is denoted by $a'$. Analogous to the derivation in Lemma \ref{lm 1}, as long as the discount factor $\gamma$ is sufficiently large, we can ensure that $Q_h(s_{t+k},a')>0$. In this case, the quantile regression for $V_h(s_{t+k})$ can be viewed as augmenting its fitting dataset with an additional value greater than zero. We now consider two possible cases
\begin{itemize}
    \item 
    When $V_{h,\text{orgin}}(s_{t+k})<=0$, the inclusion of $Q_h(s_{t+k},a')$ clearly increases the fitted value of $V_h(s_{t+k})$. By propagating this effect backward from time step $t+k$, we obtain
    \begin{equation}
        V_{h,\text{\MethodName}}(s_t)\geq V_{h,\text{orgin}}(s_t).
    \end{equation}
    \item 
    When $V_{h,\text{orgin}}(s_{t+k})>0$, we first apply Lemma~\ref{lm3}, which guarantees that incorporating $Q_h(s_{t+k},a')$ still preserves $V_h(s_{t+k})>0$. Next, we propagate this effect backward from time step $t+k$. If there exists some $t'\in [t,t+k)$, such that $V_{h,\text{orgin}}(s_{t'})<=0$, then the above modification does not affect the safety value in timestep $t'$. This is because, for a dataset whose fitted quantile is less than or equal to zero, replacing one positive value with another positive value does not change the resulting quantile. Consequently, further backward propagation from $t'$ to $t$ also has no effect, yielding
    \begin{equation}
        V_{h,\text{\MethodName}}(s_t)= V_{h,\text{orgin}}(s_t).
    \end{equation}
    Otherwise, if there exists no $t'\in [t,t+k)$ such that $V_{h,\text{orgin}}(s_{t'})<=0$, then during the backward propagation from $t+k$ to $t$, we can repeatedly apply Lemma~\ref{lm4}, which yields
    \begin{equation}
        V_{h,\text{\MethodName}}(s_t)>0.
    \end{equation}
\end{itemize}
Therefore, for all $s_t \in \mathcal{D}$, we have
\begin{equation}
    V_{h,\text{\MethodName}}(s_t)\geq V_{h,\text{orgin}}(s_t)\ \text{or}\ V_{h,\text{\MethodName}}(s_t)>0.
\end{equation}
Taken together, and noting that policy updates rely solely on the original offline dataset $\mathcal{D}$, incorporating \MethodName into safe policy learning neither affects the reward value estimation of the policy nor degrades safety performance due to underestimation of the safety value, compared to learning without \MethodName.
\end{proof}
\textbf{Remark.} 
Theorem~\ref{th_3} indicates that \MethodName can be integrated with existing offline safe reinforcement learning algorithms without making any assumptions about the number of unsafe samples in the dataset. Moreover, regardless of the magnitude of the dynamics model error, such integration does not lead to degraded policy safety due to the underestimation of the feasible value function.
The accuracy of the dynamics model primarily influences the types of samples that are conservatively evaluated in real time. When the dynamics model is sufficiently accurate, the samples for which \MethodName induces more conservative decisions tend to correspond to genuinely infeasible or risky states. In contrast, when the dynamics model is less accurate, the method may behave conservatively on some feasible samples, resulting in overly conservative policies and degraded reward performance, or may fail to apply conservative estimates on infeasible samples, potentially affecting safety performance. Nevertheless, even in cases where safety performance is impacted, the preceding analysis guarantees that the policy will not perform worse in terms of safety than a policy learned without \MethodName.

\section{More Related Work}
\label{more related work}
\paragraph{Offline RL.}
Offline RL trains policies using pre-collected datasets, avoiding real-world trial and error, which is critical for deploying RL in practical settings. Its primary challenge is addressing the extrapolation error~\citep{prudencio2023survey}. 
Methods such as CQL~\citep{kumar2020conservative} and ICQ~\citep{yang2021believe} penalize the value function of unseen actions in order to constrain actions to those observed in the offline dataset, while BCQ~\citep{fujimoto2019off} explicitly restricts candidate actions to remain close to the offline data distribution by leveraging a CVAE. Furthermore, approaches like IQL~\citep{kostrikovoffline} and DT~\citep{chen2021decision} adopt BC-based supervised learning to fully confine the policy distribution within the offline dataset distribution, thereby avoiding the effect of extrapolation error.
Others, such as MOReL~\citep{morel}, MOPO~\citep{mopo}, and MOBILE~\citep{sun2023model}, learn the environment models from the offline data and utilize these models with uncertainty estimates to avoid OOD regions with low model accuracy.

\paragraph{Generative model assisted RL.}
Driven by the powerful capacity of generative models to capture multimodal distributions, as well as their rapid progress in domains like natural language processing, recent years have witnessed a growing interest in applying such models to reinforcement learning~\citep{li2025generative}.
Apart from the aforementioned applications, LLMs can also serve other roles in RL, such as information representation and processing~\citep{paischer2022history}, action space compression~\citep{yanefficient}, world modeling~\citep{ge2024worldgpt}, and multi-agent task allocation~\citep{kannan2024smart}.
In addition to LLMs, other generative models are gaining growing traction in RL. Diffusion models, in particular, have emerged as the most widely applied due to their strong expressive power. They have been utilized for planning~\citep{janner2022planning,ajayconditional}, for offline policy learning~\citep{wangdiffusion,chi2023diffusion,dingconsistency}, and more recently for online policy learning~\citep{ding2024diffusion,celikdime,ma2025efficient}.
Flow matching, due to its simpler generative process compared to diffusion models, has also recently attracted increasing attention in RL~\citep{park2025flow,espinosa2025scaling,lv2025flow}.


\section{Implementation Details}
\label{imp detail}
In this section, we will offer the implementation details of \MethodName.
\subsection{Model Learning}
We begin with model learning using the safe-only dataset $\mathcal{D}$. Specifically, each model $T\in\hat{T}$ is formulated as a Gaussian distribution over the next state given the current state and action:
\begin{equation}
    T(s'|s,a)=\mathcal{N}(\mu(s,a),\Sigma(s,a)).
\end{equation}
Then, the ensemble dynamics model is optimized by
\begin{equation}
\label{model_update}
    \max_{\hat{T}}\mathbb{E}_{(s,a,s')\sim\mathcal{D}}[\log \hat{T}(s'|s,a)].
\end{equation}
In training, seven candidate models $T$ are trained to construct the ensemble dynamics model $\hat{T}$, and the best 5 models are picked based on the validation prediction error on a held-out set that contains 20\% samples in $\mathcal{D}$. During model rollout, we randomly pick one dynamics model from the best 5 models to obtain $s'$.
For the cost label, given the cost function $\bar{c}$, to mitigate the impact of model uncertainty and obtain a more conservative estimate, we label a sample as $1$ if any $s'$ of the five best models' outputs satisfies $\bar{c}(s')=1$.

\subsection{Conservative Cost Function Generation}
The key to LLM-based cost function generation is the formulation of the prompt input. Accordingly, this section details the prompt design employed in our approach.

First, we present the overall instruction used when initially employing the LLM for cost function generation
\begin{promptbox}{Instruction}
You are an expert python coder. Please help me write some functions to evaluate whether a given observation is safe or not. The function should return 0 if the observation is safe, and 1 otherwise. The function should be called get\_cost(observation). 

The \textbf{cost description} is: \{Cost Description\}.

The \textbf{observation description} is: \{Observation Description\}.

Now, please give the function get\_cost, you should be a little \textbf{more conservative}, some situations that do not violate safety constraints but are close to doing so can also be considered unsafe.
\end{promptbox}
Subsequently, we present two examples to demonstrate the Cost Description, Observation Description within the instruction above, and Feedback provided for the generated cost function.

\paragraph{Push2}
First, the Cost Description for the Push2 task is composed as follows:
\begin{promptbox}{Push2 Cost Description}
If you are within 0.3 meters of hazard, it is unsafe. If you contact with a pillar (being very close to the pillar), it is unsafe.
\end{promptbox}
The corresponding Observation Description is as follows:
\begin{promptbox}{Push2 Observation Description}
The observation contains some agent specific dimensions plus 4 radar vetors, the last 16 dimensions is the push box radar, the 16 dimensions before push box radar is the pillar radar, the 16 dimensions before pillar radar is the hazard radar. Each radar has 16 dimensions, meaning taking the robot's forward direction as 0 degrees, determine whether a detected target is present at every 22.5-degree interval around it. However, there is no information about robot's forward direction in the observation. The max radar detection range is 3 meters. If in a given direction, the detected target is d meters away, and d$\leq$3, then the radar will give the number as (3-d)/3. Meanwhile, since the observation lacks a collision sensor, contact with a target does not result in a radar reading of zero distance. This discrepancy is due to the physical volume of the target, whose radius generally does not exceed 0.3 meters.
\end{promptbox}
After incorporating the above content into the instruction and inputting it into the LLM (GPT o4-mini), it first generated the following cost function:
\begin{promptbox}{Push2 Output Function1}
\begin{Verbatim}[fontsize=\small, frame=single]
def get_cost(observation):
    """
    Returns 0 if the observation is safe, 1 if unsafe.

    Parameters
    ----------
    observation : array‐like, length >= 48
        The last 48 entries are three 16‐dimensional radars:
          - hazard radar:  obs[-48:-32]
          - pillar radar:  obs[-32:-16]
          - push‐box radar: obs[-16:]
        Radar values r correspond to distances d = 3*(1 - r).

    Safety rules (conservative):
      1. Any hazard reading r >= 0.90  →  unsafe
      2. Any pillar reading r >= 0.85  →  unsafe

    Returns
    -------
    int
        0 if safe, 1 if unsafe.
    """
    obs = np.asarray(observation)
    if obs.ndim != 1 or obs.size < 48:
        raise ValueError("Observation must be a 1D array with length >= 48.")

    # extract the 16‐dim hazard and pillar readings
    hazard_r = obs[-48:-32]
    pillar_r = obs[-32:-16]

    # thresholds
    HAZARD_THRESH = 0.90  # exact 0.3 m
    PILLAR_THRESH = 0.85  # small margin beyond contact

    # check if any reading exceeds threshold
    too_close_to_hazard = np.any(hazard_r >= HAZARD_THRESH)
    contact_with_pillar = np.any(pillar_r >= PILLAR_THRESH)

    return 1 if (too_close_to_hazard or contact_with_pillar) else 0
\end{Verbatim}
\end{promptbox}
However, upon evaluation, this cost function labeled 2\% of the samples in $\mathcal{D}$ as unsafe, which does not meet our conservative requirement of labeling samples within $[p_{\text{min}}, p_{\text{max}}]$ as unsafe, where $p_{\text{min}}=10\%$ and $p_{\text{max}}=30\%$ in practice.
Therefore, we provided the following feedback:
\begin{promptbox}{Push2 Feedback1}
For safe samples, this function classifies 2\% of them as unsafe. We want to classify 10\%-30\% safe samples as unsafe. It should be a little more conservative.
\end{promptbox}
Based on the above feedback, the LLM made corresponding modifications to the output function
\begin{promptbox}{Push2 Output Function2}
\begin{Verbatim}[fontsize=\small, frame=single]
def get_cost(observation):
    """
    Returns 0 if the observation is safe, 1 if unsafe (conservatively).

    Parameters
    ----------
    observation : array‐like, length >= 48
        The last 48 entries are:
          - hazard radar:  obs[-48:-32]
          - pillar radar:  obs[-32:-16]
          - push‐box radar: obs[-16:]

    Conservative thresholds:
      - Hazard: r >= 0.80 → d <= 0.60 m (unsafe)
      - Pillar: r >= 0.75 → d <= 0.75 m (unsafe)

    Returns
    -------
    int
        0 if safe, 1 if unsafe.
    """
    obs = np.asarray(observation)
    if obs.ndim != 1 or obs.size < 48:
        raise ValueError("Observation must be a 1D array with length >= 48.")

    hazard_r = obs[-48:-32]
    pillar_r = obs[-32:-16]

    HAZARD_THRESH = 0.80  # ~0.60 m
    PILLAR_THRESH = 0.75  # ~0.75 m

    too_close_to_hazard = np.any(hazard_r >= HAZARD_THRESH)
    close_to_pillar = np.any(pillar_r >= PILLAR_THRESH)

    return 1 if (too_close_to_hazard or close_to_pillar) else 0
\end{Verbatim}
\end{promptbox}
Nevertheless, the revised cost function proved excessively conservative, labeling 40\% of safe samples as unsafe. Consequently, we supplied the LLM with additional feedback:
\begin{promptbox}{Push2 Feedback2}
 For safe samples, this function classifies 40\% of them as unsafe. We want to classify 10\%-30\% safe samples as unsafe. It is too conservative.
\end{promptbox}
Finally, the LLM adjusted the cost function once again, yielding:
\begin{promptbox}{Push2 Output Function3}
\begin{Verbatim}[fontsize=\small, frame=single]
def get_cost(observation):
    """
    Returns 0 if the observation is safe, 1 if unsafe (moderately conservative).

    Parameters
    ----------
    observation : array‐like, length >= 48
        The last 48 entries are:
          - hazard radar:  obs[-48:-32]
          - pillar radar:  obs[-32:-16]
          - push‐box radar: obs[-16:]

    Thresholds tuned for 10–30% false positives on safe samples:
      - Hazard: r >= 0.85 → d <= 0.45 m (unsafe)
      - Pillar: r >= 0.80 → d <= 0.60 m (unsafe)

    Returns
    -------
    int
        0 if safe, 1 if unsafe.
    """
    obs = np.asarray(observation)
    if obs.ndim != 1 or obs.size < 48:
        raise ValueError("Observation must be a 1D array with length >= 48.")

    hazard_r = obs[-48:-32]
    pillar_r = obs[-32:-16]

    HAZARD_THRESH = 0.85  # ~0.45 m
    PILLAR_THRESH = 0.80  # ~0.60 m

    too_close_to_hazard = np.any(hazard_r >= HAZARD_THRESH)
    close_to_pillar = np.any(pillar_r >= PILLAR_THRESH)

    return 1 if (too_close_to_hazard or close_to_pillar) else 0
\end{Verbatim}
\end{promptbox}

\paragraph{Walker2dVelocity}
In MuJoCo Velocity tasks, the agent is deemed unsafe once its instantaneous velocity surpasses a predefined threshold. Nevertheless, the velocity included in the agent’s observation reflects only the terminal velocity at the end of a step, which may deviate from the velocity during step execution. Consequently, an LLM-generated cost function may align well with the textual description, or even adopt a more conservative stance, but still fail to achieve perfect classification accuracy on the 100 samples in $\mathcal{D}_{\text{unsafe}}$. This example illustrates such a case.
First, the Cost Description for the Walker2dVelocity task is composed as follows:
\begin{promptbox}{Walker2dVelocity Cost Description}
If the velocity of the robot exeeds 2.3415, it is unsafe. When calculating the velocity, we only consider x-coordinate and y-coordinate (if exists).
\end{promptbox}
The corresponding Observation Description is as follows:
\begin{promptbox}{Walker2dVelocity Observation Description}
The observation contains 17 dimensions. The first is z-coordinate of the front tip; the second is angle of the torso; the third is angle of the thigh joint; the forth is angle of the leg joint; the fifth is angle of the foot joint; the sixth is angle of the left thigh joint; the seventh is angle of the left leg joint; the eighth is angle of the left foot joint; the nineth is velocity of the x-coordinate of the torso; the tenth is velocity of the z-coordinate (height) of the torso; the eleventh is angular velocity of the angle of the torso; the twelveth is angular velocity of the thigh hinge; the thirteenth is angular velocity of the leg hinge; the forteenth is angular velocity of the foot hinge; the fifteenth is angular velocity of the thigh hinge; the sixteenth is angular velocity of the leg hinge; the seventeenth is angular velocity of the foot hinge.
\end{promptbox}
After incorporating the above content into the instruction and inputting it into the LLM, it first generated the following cost function:
\begin{promptbox}{Walker2dVelocity Output Function1}
\begin{Verbatim}[fontsize=\small, frame=single]
def get_cost(observation, max_safe_vel=2.3415, safety_margin=0.05):
    """
    Return 0 if the observation is safe, 1 otherwise.

    We consider only the torso's x-velocity (observation[8]).
    Any |velocity| >= (1 - safety_margin) * max_safe_vel is deemed unsafe.

    Parameters
    ----------
    observation : sequence of float, length >= 9
        The 17-dimensional observation vector.
    max_safe_vel : float
        The hard safety limit on velocity.
    safety_margin : float
        Fractional margin to be conservative (e.g. 0.05 means 5% below max_safe_vel).

    Returns
    -------
    int
        0 if safe, 1 if unsafe.
    """
    # pick out the x-velocity of the torso
    v_x = observation[8]

    # define the conservative threshold
    threshold = (1.0 - safety_margin) * max_safe_vel

    # unsafe if we're at or above the threshold
    return 1 if abs(v_x) >= threshold else 0
\end{Verbatim}
\end{promptbox}
However, the accuracy of this function on $\mathcal{D}_{\text{unsafe}}$ is only 98\%. Therefore, we provided feedback to the LLM:
\begin{promptbox}{Walker2dVelocity Feedback1}
For 100 unsafe testing samples, this function achieves 98\% accuracy, it should be a little more conservative.
\end{promptbox}
Based on the above feedback, the LLM made corresponding modifications to the output function
\begin{promptbox}{Walker2dVelocity Output Function2}
\begin{Verbatim}[fontsize=\small, frame=single]
def get_cost(observation, max_safe_vel=2.3415, safety_margin=0.10):
    """
    Return 0 if the observation is safe, 1 otherwise.

    Unsafe if torso's x-velocity exceeds (1 - safety_margin) * max_safe_vel.
    This version is more conservative, using a 10% margin.

    Parameters
    ----------
    observation : sequence of float, length >= 9
        The 17-dimensional observation vector.
    max_safe_vel : float
        The hard velocity safety threshold.
    safety_margin : float
        A conservative margin to catch near-threshold values.

    Returns
    -------
    int
        0 if safe, 1 if unsafe.
    """
    v_x = observation[8]
    threshold = (1.0 - safety_margin) * max_safe_vel
    return 1 if abs(v_x) >= threshold else 0
\end{Verbatim}
\end{promptbox}
The accuracy of this function on $\mathcal{D}_{\text{unsafe}}$ remains at only 99\%. Therefore, we provided additional feedback
\begin{promptbox}{Walker2dVelocity Feedback2}
 For 100 unsafe testing samples, this function achieves 99\% accuracy, it should be a little more conservative.
\end{promptbox}
Finally, the LLM adjusted the cost function once again, yielding:
\begin{promptbox}{Walker2dVelocity Output Function3}
\begin{Verbatim}[fontsize=\small, frame=single]
def get_cost(observation, max_safe_vel=2.3415, safety_margin=0.15:
    """
    Return 0 if the observation is safe, 1 otherwise.

    Unsafe if torso’s x-velocity magnitude exceeds
    (1 - safety_margin) * max_safe_vel.

    This default uses a 15% margin (i.e. flags |v_x| >= ~1.9903).

    Parameters
    ----------
    observation : sequence of float, len >= 9
        The 17-dimensional observation vector.
    max_safe_vel : float
        The hard velocity safety limit.
    safety_margin : float
        Fractional buffer under max_safe_vel (default 0.15).

    Returns
    -------
    int
        0 if safe, 1 if unsafe.
    """
    # torso x-velocity is at index 8
    v_x = observation[8]

    # build a conservative threshold
    threshold = (1.0 - safety_margin) * max_safe_vel

    return 1 if abs(v_x) >= threshold else 0
\end{Verbatim}
\end{promptbox}

\subsection{Model Rollout}
During model rollout, we follow the rollout procedure used in works such as MOPO, performing a rollout after every $K$ policy gradient updates. For each rollout, a batch of size $b$ is sampled from $\mathcal{D}$, and each data point in the batch is rolled out for $H$ steps. The process of sampling and branch rollout is repeated $E$ times.
In practice, we set $K=250000$, $b=50000$, $H=1$, and $E=10$.

        
\subsection{Hyperparameters}
The training of \MethodName involves the selection of hyperparameters. To ensure reproducibility, this section outlines the specific hyperparameters used in our experiments, as shown in \cref{hyper}.
Specifically, for the model rollout frequency, we set it to $2.5e4$ for LSPC and CAPS, and to $2.5e5$ for FISOR.
\MethodName is implemented based on OSRL~\footnote{\href{https://github.com/liuzuxin/OSRL}{https://github.com/liuzuxin/OSRL}}, FISOR~\footnote{\href{https://github.com/ZhengYinan-AIR/FISOR}{https://github.com/ZhengYinan-AIR/FISOR}}, LSPC~\footnote{\href{https://github.com/PrajwalKoirala/LSPC-Safe-Offline-RL}{https://github.com/PrajwalKoirala/LSPC-Safe-Offline-RL}}, and CAPS~\footnote{\href{https://github.com/yassineCh/CAPS}{https://github.com/yassineCh/CAPS}} code bases, and the default parameters are retained except for FISOR, we set the expectile $\tau$ to 0.9 in Navigation tasks, and to $0.95$ in the Velocity tasks. Also, we set the model training step to $2e6$ in FISOR.

\begin{table}[ht]
\small
\renewcommand{\arraystretch}{1.2}
    \centering
    \caption{Hyperparameter choices of \MethodName.}
    \label{hyper}
    \begin{tabular}{c|c|c}
        \toprule
        \multicolumn{1}{c|}{} &
        Hyperparameter  & Value \\
        \midrule
        \multirow{6}{*}{model learning} &
        model hidden layers  & $[256,256,256,256]$ \\ &
        layer weight decays & $[2.5e-5,5e-5,7.5e-5,7.5e-5,1e-4]$\\ &
        model ensemble number & $7$\\ &
        model elites number & $5$\\ &
        batch size & $512$\\ &
        learning rate & $0.001$\\
        \midrule
         \multirow{3}{*}{cost function generation} &
        $p_{\text{min}}$  & $10\%$ \\ &
        $p_{\text{max}}$  & $30\%$\\ &
        max LLM query number & $10$\\
        \midrule
        \multirow{5}{*}{policy learning} &
        model rollout frequency & $\{2.5e4, 2.5e5\}$\\ &
        model rollout batch size & $50000$\\ &
        model rollout length & $1$\\ &
        model rollout epoch & $10$\\ &
        model rollout std & $0.1$\\
        \bottomrule
    \end{tabular}
\end{table}

\section{Detailed Description of the Tasks and Baselines}
\label{detail task base}
\subsection{Tasks}
\begin{figure}[t]
    \centering
    \includegraphics[width=0.99\textwidth]{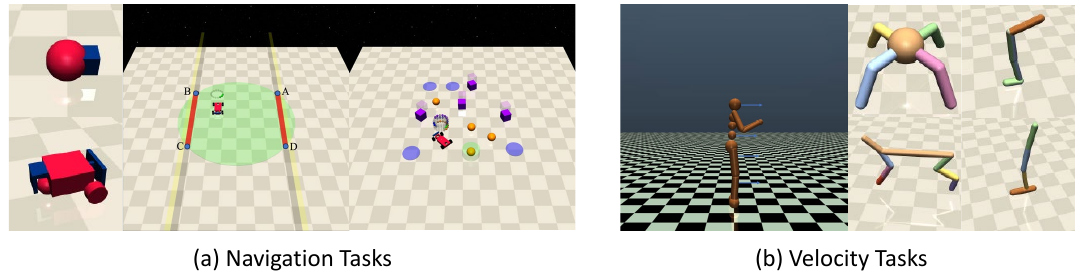}
    \caption{Tasks used in this paper. (a) Navigation Tasks based on Point and Car robots. (b) Velocity Tasks based on Ant, HalfCheetah, Hopper, Walker2d, and Swimmer robots.}
    \label{tasks}
\end{figure}

All tasks used in this paper are derived from the Safety-Gymnasium's Navigation Tasks and Velocity Tasks. In the Navigation Tasks, there are two different types of robots: Point and Car, which we need to control to navigate through the environment and earn rewards by reaching target points, pressing the correct buttons, or moving in designated directions. Different tasks also have varying costs, such as avoiding collisions with specific targets, preventing incorrect button presses, and staying within designated boundaries.

The Velocity Tasks are built on traditional MuJoCo~\citep{todorov2012mujoco} simulations, requiring robots such as Ant, HalfCheetah, Swimmer, and Walker2d to move, where higher speeds result in higher rewards. However, each robot has specific safety velocity thresholds for different tasks, and exceeding these thresholds leads to unsafe states. For detailed descriptions of each task, refer to the original Safety-Gymnasium paper~\citep{ji2023safety}. 

\subsection{Baselines}

We provide a more detailed introduction to the baselines in this section.
\begin{itemize}
    \item \textbf{CDT} is the previous advanced algorithm under sequence modeling for offline safe RL. After incorporating CTG into DT, CDT seeks to address the conflict between safety constraints and reward maximization. To tackle this issue, CDT proposes a data augmentation approach, where reward returns for certain safe but low-reward trajectories in the offline dataset are re-labeled with higher values. 
    \item \textbf{CCAC} is one of the advanced algorithms for soft-constraint offline safe RL. It employs conservative estimation to obtain pessimistic value estimates and incorporates the constraint condition into the policy input, enabling the policy to handle multiple constraint thresholds simultaneously.
    \item \textbf{FISOR} is one of the advanced offline safe RL algorithm based on hard-constraint modeling. Similar to other algorithms under hard-constraints~\cite{yu2022reachability, ganai2023iterative}, it first divides the state space into feasible and infeasible regions. It then uses the IQL~\cite{kostrikovoffline} algorithm to learn feasible value functions for offline feasible region identification. Next, FISOR sets distinct learning objectives for the feasible and infeasible regions. In the feasible region, it aims to maximize the reward while ensuring feasibility, while in the infeasible region, it focuses on minimizing constraint violations. Finally, FISOR employs a diffusion model to represent and learn the policy.
    \item \textbf{LSPC}-O, similar to FISOR, performs hard-constraint safe policy learning. It first trains a CVAE using the feasibility advantage as a weighting factor, such that the CVAE generates only safe actions. Subsequently, optimization is conducted in the latent space of the CVAE with the objective of maximizing reward, aiming to achieve a balance between reward optimization and safety preservation.
    \item \textbf{CAPS}(IQL) is also an advanced algorithm under the soft-constraint setting. Unlike methods that explicitly incorporate constraint thresholds during training, CAPS does not use the constraint threshold in its training process. Instead, similar to FISOR, it adopts IQL-style expectile regression for value function learning and employs weighted behavior cloning for policy learning. To enable decision-making under different constraint thresholds at deployment time, CAPS utilizes a multi-head policy that outputs multiple candidate actions, from which an action satisfying the specified constraint threshold and achieving high reward is selected.
\end{itemize}

\section{More Experimental Results}
\label{more results}
\subsection{Time Complexity Analysis}
To better understand the practical utility of \MethodName, we evaluate its training time overhead in this section. Using the CarButton environment as an example, we report the training time measured on a single NVIDIA 4090 GPU under an otherwise idle setting (i.e., with no additional CPU/GPU workload), as summarized in \cref{time}.

\begin{table}[ht]
\small
\renewcommand{\arraystretch}{1.2}
    \centering
    \caption{Time complexity analysis.}
    \label{time}
    \resizebox{\textwidth}{!}{
\begin{tabular}{c|c|c|c}
\hline
Dynamics Learning & \makecell{Maximum LLM Generation\\(reflection for full 10 rounds)} & \makecell{Total Dynamics Rollout\\(including rollout generation and\\ labeling over all 10 rollout epochs)} & Policy Training \\ \hline
0.58h & 0.20h & 0.03h & 1.88h \\ \hline
\end{tabular}
    }
\end{table}

\subsection{Sensitivity on LLM Reflection Rounds}
In this section, we investigate the impact of different maximum LLM reflection rounds on policy performance. The results are presented in \cref{round_sens}.
As shown, increasing the maximum number of LLM reflection rounds leads to a notable improvement in the safety performance of the policy, further highlighting the importance of the reflection mechanism in enhancing the reliability of LLMs.
\begin{table}[h!]
\small
\renewcommand{\arraystretch}{1.3}
    \centering
    \caption{Sensitive analysis on different LLM reflection rounds.}
    \label{round_sens}
\begin{tabular}{c|cccccccc}
\hline
                       & \multicolumn{2}{c}{0}                                      & \multicolumn{2}{c}{1}                                      & \multicolumn{2}{c}{5}                                      & \multicolumn{2}{c}{10}                                    \\
\multirow{-2}{*}{Task} & r↑                           & c↓                          & r↑                           & c↓                          & r↑                           & c↓                          & r↑                          & c↓                          \\ \hline
PointButton1           & {\color[HTML]{9B9B9B} 0.06}  & {\color[HTML]{9B9B9B} 3.40} & {\color[HTML]{9B9B9B} 0.06}  & {\color[HTML]{9B9B9B} 2.87} & {\color[HTML]{9B9B9B} 0.02}  & {\color[HTML]{9B9B9B} 1.46} & {\color[HTML]{9B9B9B} 0.01} & {\color[HTML]{9B9B9B} 1.15} \\
PointButton2           & {\color[HTML]{9B9B9B} 0.07}  & {\color[HTML]{9B9B9B} 4.55} & {\color[HTML]{9B9B9B} 0.07}  & {\color[HTML]{9B9B9B} 3.39} & {\color[HTML]{9B9B9B} 0.05}  & {\color[HTML]{9B9B9B} 2.11} & {\color[HTML]{9B9B9B} 0.08} & {\color[HTML]{9B9B9B} 2.76} \\
PointGoal1             & \textbf{0.34}                & \textbf{0.86}               & {\color[HTML]{9B9B9B} 0.50}  & {\color[HTML]{9B9B9B} 1.68} & {\color[HTML]{9B9B9B} 0.43}  & {\color[HTML]{9B9B9B} 1.35} & \textbf{0.35}               & \textbf{0.96}               \\
PointGoal2             & \textbf{0.08}                & \textbf{0.42}               & \textbf{0.08}                & \textbf{0.52}               & \textbf{0.09}                & \textbf{0.61}               & \textbf{0.08}               & \textbf{0.38}               \\
PointPush1             & {\color[HTML]{9B9B9B} 0.20}  & {\color[HTML]{9B9B9B} 2.17} & \textbf{0.20}                & \textbf{0.53}               & \textbf{0.14}                & \textbf{0.53}               & \textbf{0.17}               & \textbf{0.86}               \\
PointPush2             & {\color[HTML]{9B9B9B} 0.07}  & {\color[HTML]{9B9B9B} 1.22} & {\color[HTML]{9B9B9B} 0.12}  & {\color[HTML]{9B9B9B} 1.47} & {\color[HTML]{9B9B9B} 0.07}  & {\color[HTML]{9B9B9B} 1.48} & \textbf{0.12}               & \textbf{0.30}               \\
CarButton1             & {\color[HTML]{9B9B9B} -0.01} & {\color[HTML]{9B9B9B} 2.57} & {\color[HTML]{9B9B9B} -0.01} & {\color[HTML]{9B9B9B} 2.54} & \textbf{-0.05}               & \textbf{0.64}               & \textbf{-0.03}              & \textbf{0.60}               \\
CarButton2             & {\color[HTML]{9B9B9B} 0.00}  & {\color[HTML]{9B9B9B} 3.21} & {\color[HTML]{9B9B9B} -0.01} & {\color[HTML]{9B9B9B} 3.94} & {\color[HTML]{9B9B9B} -0.03} & {\color[HTML]{9B9B9B} 1.26} & \textbf{0.00}               & \textbf{0.99}               \\
CarGoal1               & \textbf{0.16}                & \textbf{0.16}               & \textbf{0.12}                & \textbf{0.03}               & \textbf{0.15}                & \textbf{0.11}               & \textbf{0.13}               & \textbf{0.02}               \\
CarGoal2               & \textbf{0.04}                & \textbf{0.80}               & \textbf{0.06}                & \textbf{0.36}               & \textbf{0.02}                & \textbf{0.37}               & \textbf{0.03}               & \textbf{0.52}               \\
CarPush1               & {\color[HTML]{9B9B9B} 0.11}  & {\color[HTML]{9B9B9B} 1.86} & {\color[HTML]{9B9B9B} 0.16}  & {\color[HTML]{9B9B9B} 1.41} & \textbf{0.16}                & \textbf{0.81}               & \textbf{0.16}               & \textbf{0.93}               \\
CarPush2               & {\color[HTML]{9B9B9B} 0.03}  & {\color[HTML]{9B9B9B} 1.19} & {\color[HTML]{9B9B9B} 0.05}  & {\color[HTML]{9B9B9B} 2.64} & {\color[HTML]{9B9B9B} 0.05}  & {\color[HTML]{9B9B9B} 3.77} & \textbf{0.01}               & \textbf{0.04}               \\
SwimmerVelocityV1      & \textbf{0.00}                & \textbf{0.23}               & \textbf{0.02}                & \textbf{0.00}               & \textbf{0.04}                & \textbf{0.00}               & \textbf{0.02}               & \textbf{0.12}               \\
HopperVelocityV1       & \textbf{0.13}                & \textbf{0.00}               & \textbf{0.09}                & \textbf{0.05}               & \textbf{0.25}                & \textbf{0.00}               & \textbf{0.11}               & \textbf{0.07}               \\
HalfCheetahVelocityV1  & \textbf{0.65}                & \textbf{0.00}               & \textbf{0.41}                & \textbf{0.00}               & \textbf{0.46}                & \textbf{0.00}               & \textbf{0.48}               & \textbf{0.00}               \\
Walker2dVelocityV1     & {\color[HTML]{9B9B9B} 0.18}  & {\color[HTML]{9B9B9B} 2.09} & {\color[HTML]{9B9B9B} 0.17}  & {\color[HTML]{9B9B9B} 1.57} & \textbf{0.13}                & \textbf{0.83}               & \textbf{0.09}               & \textbf{0.58}               \\
AntVelocityV1          & \textbf{0.50}                & \textbf{0.00}               & \textbf{0.55}                & \textbf{0.00}               & \textbf{0.45}                & \textbf{0.00}               & \textbf{0.52}               & \textbf{0.00}               \\ \hline
Average                & {\color[HTML]{9B9B9B} 0.15}  & {\color[HTML]{9B9B9B} 1.45} & {\color[HTML]{9B9B9B} 0.16}  & {\color[HTML]{9B9B9B} 1.35} & \textbf{0.14}                & \textbf{0.90}               & \textbf{0.14}               & \textbf{0.60}               \\ \hline
\end{tabular}
\end{table}

\subsection{Discussion about Prompt Wording}
Prompt wording is indeed a key factor influencing the effectiveness of LLMs. Therefore, in this section, we provide a detailed discussion on prompt design. Our LLM prompt consists of three main components (excluding reflection): the task prompt, the cost prompt, and the instruction prompt. Since the task prompt and cost prompt contain essential information, they are difficult to modify or omit. Consequently, the primary design focus lies in the instruction prompt.

Within the instruction prompt, in addition to specifying the requirement to generate cost signals, the design mainly involves two elements: a conservative prompt (“you should be a little more conservative, situations that do not violate safety constraints but are close to doing so can also be considered unsafe”) and a role-playing prompt (“you are an expert Python coder”). We have already analyzed the effect of the conservative prompt in the ablation study. Here, we further introduce a new baseline, \textit{W/o Role}, to evaluate the necessity of the role-playing prompt. 
As shown in \cref{prompt_sens}, the specially designed conservative prompt, the same as in ablation study, has a noticeable impact on the safety performance of the policy. In contrast, the role-playing prompt leads to only marginal performance differences, indicating that, when the LLM is sufficiently capable, our method is relatively robust to non-informational prompt variations.

\begin{table}[h!]
\small
\renewcommand{\arraystretch}{1.3}
    \centering
    \caption{Sensitive analysis on prompt wording.}
    \label{prompt_sens}
\begin{tabular}{c|cccccc}
\hline
                       & \multicolumn{2}{c}{PROCO}                                 & \multicolumn{2}{c}{W/o Consv.}                             & \multicolumn{2}{c}{W/o Role}                              \\
\multirow{-2}{*}{Task} & r↑                          & c↓                          & r↑                           & c↓                          & r↑                          & c↓                          \\ \hline
PointButton1           & {\color[HTML]{9B9B9B} 0.01} & {\color[HTML]{9B9B9B} 1.15} & {\color[HTML]{9B9B9B} 0.09}  & {\color[HTML]{9B9B9B} 4.01} & {\color[HTML]{9B9B9B} 0.03} & {\color[HTML]{9B9B9B} 1.54} \\
PointButton2           & {\color[HTML]{9B9B9B} 0.08} & {\color[HTML]{9B9B9B} 2.76} & {\color[HTML]{9B9B9B} 0.08}  & {\color[HTML]{9B9B9B} 3.54} & {\color[HTML]{9B9B9B} 0.06} & {\color[HTML]{9B9B9B} 2.01} \\
PointGoal1             & \textbf{0.35}               & \textbf{0.96}               & \textbf{0.37}                & \textbf{0.63}               & {\color[HTML]{9B9B9B} 0.42} & {\color[HTML]{9B9B9B} 1.35} \\
PointGoal2             & \textbf{0.08}               & \textbf{0.38}               & \textbf{0.06}                & \textbf{0.74}               & \textbf{0.06}               & \textbf{0.93}               \\
PointPush1             & \textbf{0.17}               & \textbf{0.86}               & \textbf{0.22}                & \textbf{0.94}               & \textbf{0.18}               & \textbf{0.45}               \\
PointPush2             & \textbf{0.12}               & \textbf{0.30}               & {\color[HTML]{9B9B9B} 0.11}  & {\color[HTML]{9B9B9B} 3.98} & {\color[HTML]{9B9B9B} 0.19} & {\color[HTML]{9B9B9B} 1.79} \\
CarButton1             & \textbf{-0.03}              & \textbf{0.60}               & {\color[HTML]{9B9B9B} -0.04} & {\color[HTML]{9B9B9B} 2.76} & \textbf{-0.01}              & \textbf{0.53}               \\
CarButton2             & \textbf{0.00}               & \textbf{0.99}               & {\color[HTML]{9B9B9B} -0.01} & {\color[HTML]{9B9B9B} 2.40} & \textbf{-0.03}              & \textbf{0.99}               \\
CarGoal1               & \textbf{0.13}               & \textbf{0.02}               & \textbf{0.23}                & \textbf{0.50}               & \textbf{0.11}               & \textbf{0.09}               \\
CarGoal2               & \textbf{0.03}               & \textbf{0.52}               & \textbf{0.03}                & \textbf{0.39}               & \textbf{0.04}               & \textbf{0.14}               \\
CarPush1               & \textbf{0.16}               & \textbf{0.93}               & \textbf{0.22}                & \textbf{0.45}               & \textbf{0.13}               & \textbf{0.78}               \\
CarPush2               & \textbf{0.01}               & \textbf{0.04}               & {\color[HTML]{9B9B9B} 0.04}  & {\color[HTML]{9B9B9B} 3.00} & \textbf{0.02}               & \textbf{0.31}               \\
SwimmerVelocityV1      & \textbf{0.02}               & \textbf{0.12}               & \textbf{0.01}                & \textbf{0.00}               & \textbf{0.03}               & \textbf{0.00}               \\
HopperVelocityV1       & \textbf{0.11}               & \textbf{0.07}               & \textbf{0.06}                & \textbf{0.00}               & \textbf{0.18}               & \textbf{0.02}               \\
HalfCheetahVelocityV1  & \textbf{0.48}               & \textbf{0.00}               & \textbf{0.81}                & \textbf{0.00}               & \textbf{0.48}               & \textbf{0.00}               \\
Walker2dVelocityV1     & \textbf{0.09}               & \textbf{0.58}               & {\color[HTML]{9B9B9B} 0.22}  & {\color[HTML]{9B9B9B} 2.28} & \textbf{0.13}               & \textbf{0.66}               \\
AntVelocityV1          & \textbf{0.52}               & \textbf{0.00}               & \textbf{0.41}                & \textbf{0.00}               & \textbf{0.50}               & \textbf{0.00}               \\ \hline
Average                & \textbf{0.14}               & \textbf{0.60}               & 0.17                         & 1.51                        & \textbf{0.15}               & \textbf{0.68}               \\ \hline
\end{tabular}
\end{table}

\subsection{Sensitivity on Unsafe Data Selection}

In the main paper, we evaluate the performance of \MethodName and FISOR under varying amounts of unsafe data. Since the selection of unsafe data is randomized, we further investigate the impact of different unsafe data selections in this section. Specifically, under the 10\% unsafe data setting, we generate two additional datasets using different random seeds. Together with the dataset used in the main paper, resulting in three distinct datasets for evaluation.
As shown in \cref{unsafe_sens}, \MethodName is largely insensitive to the specific selection of unsafe data. This is likely because our method can generate a substantial number of unsafe samples, thereby reducing the influence of the initial unsafe dataset. In contrast, FISOR exhibits higher sensitivity due to the limited amount of unsafe data available. While different unsafe subsets may introduce some variation in safety performance on individual environments, the overall average performance—and, importantly, the issue of severe safety violations—remains largely unchanged.

\begin{table}[h!]
\small
\renewcommand{\arraystretch}{1.3}
    \centering
    \caption{Sensitive analysis on different unsafe data.}
    \label{unsafe_sens}
    \resizebox{\textwidth}{!}{
\begin{tabular}{c|cccccccccccc}
\hline
                       & \multicolumn{2}{c}{Ours Rand1}                             & \multicolumn{2}{c}{FISOR Rand1}                             & \multicolumn{2}{c}{Ours Rand2}                             & \multicolumn{2}{c}{FISOR Rand2}                             & \multicolumn{2}{c}{Ours Rand3}                             & \multicolumn{2}{c}{FISOR Rand3}                             \\
\multirow{-2}{*}{Task} & r↑                           & c↓                          & r↑                           & c↓                           & r↑                           & c↓                          & r↑                           & c↓                           & r↑                           & c↓                          & r↑                           & c↓                           \\ \hline
PointButton1           & {\color[HTML]{9B9B9B} 0.05}  & {\color[HTML]{9B9B9B} 1.33} & {\color[HTML]{9B9B9B} 0.40}  & {\color[HTML]{9B9B9B} 9.75}  & {\color[HTML]{9B9B9B} 0.03}  & {\color[HTML]{9B9B9B} 1.55} & {\color[HTML]{9B9B9B} 0.41}  & {\color[HTML]{9B9B9B} 9.11}  & {\color[HTML]{9B9B9B} 0.02}  & {\color[HTML]{9B9B9B} 1.48} & {\color[HTML]{9B9B9B} 0.45}  & {\color[HTML]{9B9B9B} 11.24} \\
PointButton2           & {\color[HTML]{9B9B9B} 0.05}  & {\color[HTML]{9B9B9B} 2.24} & {\color[HTML]{9B9B9B} 0.44}  & {\color[HTML]{9B9B9B} 10.78} & {\color[HTML]{9B9B9B} 0.05}  & {\color[HTML]{9B9B9B} 2.34} & {\color[HTML]{9B9B9B} 0.42}  & {\color[HTML]{9B9B9B} 12.15} & {\color[HTML]{9B9B9B} 0.07}  & {\color[HTML]{9B9B9B} 1.52} & {\color[HTML]{9B9B9B} 0.42}  & {\color[HTML]{9B9B9B} 11.82} \\
PointGoal1             & {\color[HTML]{9B9B9B} 0.44}  & {\color[HTML]{9B9B9B} 1.63} & {\color[HTML]{9B9B9B} 0.70}  & {\color[HTML]{9B9B9B} 4.24}  & {\color[HTML]{9B9B9B} 0.41}  & {\color[HTML]{9B9B9B} 1.15} & {\color[HTML]{9B9B9B} 0.72}  & {\color[HTML]{9B9B9B} 4.80}  & {\color[HTML]{9B9B9B} 0.43}  & {\color[HTML]{9B9B9B} 1.90} & {\color[HTML]{9B9B9B} 0.70}  & {\color[HTML]{9B9B9B} 4.77}  \\
PointGoal2             & \textbf{0.07}                & \textbf{0.83}               & {\color[HTML]{9B9B9B} 0.56}  & {\color[HTML]{9B9B9B} 11.04} & \textbf{0.08}                & \textbf{0.69}               & {\color[HTML]{9B9B9B} 0.60}  & {\color[HTML]{9B9B9B} 10.76} & \textbf{0.05}                & \textbf{0.87}               & {\color[HTML]{9B9B9B} 0.60}  & {\color[HTML]{9B9B9B} 10.62} \\
PointPush1             & \textbf{0.18}                & \textbf{0.77}               & {\color[HTML]{9B9B9B} 0.35}  & {\color[HTML]{9B9B9B} 3.52}  & {\color[HTML]{9B9B9B} 0.16}  & {\color[HTML]{9B9B9B} 1.11} & {\color[HTML]{9B9B9B} 0.39}  & {\color[HTML]{9B9B9B} 2.81}  & \textbf{0.18}                & \textbf{0.74}               & {\color[HTML]{9B9B9B} 0.38}  & {\color[HTML]{9B9B9B} 3.17}  \\
PointPush2             & {\color[HTML]{9B9B9B} 0.14}  & {\color[HTML]{9B9B9B} 1.40} & {\color[HTML]{9B9B9B} 0.28}  & {\color[HTML]{9B9B9B} 6.17}  & {\color[HTML]{9B9B9B} 0.16}  & {\color[HTML]{9B9B9B} 1.60} & {\color[HTML]{9B9B9B} 0.26}  & {\color[HTML]{9B9B9B} 5.30}  & {\color[HTML]{9B9B9B} 0.13}  & {\color[HTML]{9B9B9B} 1.36} & {\color[HTML]{9B9B9B} 0.28}  & {\color[HTML]{9B9B9B} 5.73}  \\
CarButton1             & \textbf{-0.06}               & \textbf{0.51}               & {\color[HTML]{9B9B9B} 0.35}  & {\color[HTML]{9B9B9B} 15.25} & \textbf{-0.01}               & \textbf{0.79}               & {\color[HTML]{9B9B9B} 0.37}  & {\color[HTML]{9B9B9B} 15.64} & {\color[HTML]{9B9B9B} -0.02} & {\color[HTML]{9B9B9B} 1.03} & {\color[HTML]{9B9B9B} 0.35}  & {\color[HTML]{9B9B9B} 20.97} \\
CarButton2             & {\color[HTML]{9B9B9B} -0.02} & {\color[HTML]{9B9B9B} 1.21} & {\color[HTML]{9B9B9B} 0.27}  & {\color[HTML]{9B9B9B} 19.57} & {\color[HTML]{9B9B9B} -0.03} & {\color[HTML]{9B9B9B} 1.41} & {\color[HTML]{9B9B9B} 0.26}  & {\color[HTML]{9B9B9B} 16.21} & {\color[HTML]{9B9B9B} -0.03} & {\color[HTML]{9B9B9B} 1.31} & {\color[HTML]{9B9B9B} 0.24}  & {\color[HTML]{9B9B9B} 16.46} \\
CarGoal1               & \textbf{0.15}                & \textbf{0.02}               & {\color[HTML]{9B9B9B} 0.72}  & {\color[HTML]{9B9B9B} 3.89}  & \textbf{0.13}                & \textbf{0.00}               & {\color[HTML]{9B9B9B} 0.65}  & {\color[HTML]{9B9B9B} 4.53}  & \textbf{0.13}                & \textbf{0.00}               & {\color[HTML]{9B9B9B} 0.66}  & {\color[HTML]{9B9B9B} 4.42}  \\
CarGoal2               & \textbf{0.03}                & \textbf{0.14}               & {\color[HTML]{9B9B9B} 0.45}  & {\color[HTML]{9B9B9B} 10.80} & \textbf{0.01}                & \textbf{0.46}               & {\color[HTML]{9B9B9B} 0.42}  & {\color[HTML]{9B9B9B} 11.14} & \textbf{0.04}                & \textbf{0.16}               & {\color[HTML]{9B9B9B} 0.51}  & {\color[HTML]{9B9B9B} 10.67} \\
CarPush1               & \textbf{0.17}                & \textbf{0.46}               & {\color[HTML]{9B9B9B} 0.43}  & {\color[HTML]{9B9B9B} 1.89}  & \textbf{0.16}                & \textbf{0.69}               & {\color[HTML]{9B9B9B} 0.42}  & {\color[HTML]{9B9B9B} 2.93}  & \textbf{0.16}                & \textbf{0.86}               & {\color[HTML]{9B9B9B} 0.34}  & {\color[HTML]{9B9B9B} 2.22}  \\
CarPush2               & \textbf{0.03}                & \textbf{0.69}               & {\color[HTML]{9B9B9B} 0.34}  & {\color[HTML]{9B9B9B} 7.55}  & \textbf{0.05}                & \textbf{0.04}               & {\color[HTML]{9B9B9B} 0.33}  & {\color[HTML]{9B9B9B} 5.95}  & \textbf{0.04}                & \textbf{0.83}               & {\color[HTML]{9B9B9B} 0.31}  & {\color[HTML]{9B9B9B} 6.60}  \\
SwimmerVelocityV1      & \textbf{0.04}                & \textbf{0.00}               & {\color[HTML]{9B9B9B} -0.02} & {\color[HTML]{9B9B9B} 6.10}  & \textbf{0.05}                & \textbf{0.00}               & {\color[HTML]{9B9B9B} -0.05} & {\color[HTML]{9B9B9B} 7.84}  & \textbf{0.04}                & \textbf{0.01}               & {\color[HTML]{9B9B9B} -0.05} & {\color[HTML]{9B9B9B} 4.02}  \\
HopperVelocityV1       & \textbf{0.10}                & \textbf{0.01}               & {\color[HTML]{9B9B9B} 0.16}  & {\color[HTML]{9B9B9B} 1.51}  & \textbf{0.11}                & \textbf{0.00}               & {\color[HTML]{9B9B9B} 0.11}  & {\color[HTML]{9B9B9B} 2.78}  & \textbf{0.05}                & \textbf{0.00}               & {\color[HTML]{9B9B9B} 0.08}  & {\color[HTML]{9B9B9B} 3.21}  \\
HalfCheetahVelocityV1  & \textbf{0.42}                & \textbf{0.00}               & \textbf{0.92}                & \textbf{0.00}                & \textbf{0.43}                & \textbf{0.00}               & \textbf{0.95}                & \textbf{0.00}                & \textbf{0.40}                & \textbf{0.00}               & \textbf{0.92}                & \textbf{0.02}                \\
Walker2dVelocityV1     & {\color[HTML]{9B9B9B} 0.12}  & {\color[HTML]{9B9B9B} 1.15} & {\color[HTML]{9B9B9B} 0.52}  & {\color[HTML]{9B9B9B} 4.14}  & \textbf{0.13}                & \textbf{0.97}               & {\color[HTML]{9B9B9B} 0.10}  & {\color[HTML]{9B9B9B} 1.59}  & \textbf{0.09}                & \textbf{0.89}               & {\color[HTML]{9B9B9B} 0.15}  & {\color[HTML]{9B9B9B} 2.29}  \\
AntVelocityV1          & \textbf{0.49}                & \textbf{0.00}               & \textbf{0.86}                & \textbf{0.94}                & \textbf{0.48}                & \textbf{0.00}               & {\color[HTML]{9B9B9B} 0.91}  & {\color[HTML]{9B9B9B} 1.13}  & \textbf{0.45}                & \textbf{0.00}               & {\color[HTML]{9B9B9B} 0.94}  & {\color[HTML]{9B9B9B} 1.62}  \\ \hline
Average                & \textbf{0.14}                & \textbf{0.73}               & {\color[HTML]{9B9B9B} 0.46}  & {\color[HTML]{9B9B9B} 6.89}  & \textbf{0.14}                & \textbf{0.75}               & {\color[HTML]{9B9B9B} 0.43}  & {\color[HTML]{9B9B9B} 6.75}  & \textbf{0.13}                & \textbf{0.76}               & {\color[HTML]{9B9B9B} 0.43}  & {\color[HTML]{9B9B9B} 7.05}  \\ \hline
\end{tabular}
    }
\end{table}

\subsection{Integrate \MethodName with Soft-constraint Algorithms}
In this section, we further integrate \MethodName with three soft-constraint offline safe reinforcement learning algorithms—\textbf{CPQ}~\cite{xu2022constraints}, \textbf{BCQ Lagrange (BCQ-L)}~\citep{fujimoto2019off}, and \textbf{COptiDICE}~\citep{leecoptidice}—and evaluate the impact of the model-based rollout data generated by \MethodName on these methods. The results are shown in \cref{integrate others}. We observe that all three algorithms benefit from the rollout data provided by \MethodName, exhibiting clear improvements in safety performance, which validates the effectiveness of the proposed model-based rollout mechanism. Nevertheless, even after these improvements, their safety performance remains relatively unsatisfactory, further indicating that feasibility-based hard-constraint methods are more suitable for this class of scenarios.

\begin{table}[h!]
\small
\renewcommand{\arraystretch}{1.5}
    \centering
    \caption{Results of integrate \MethodName with soft-constraint algorithms.}
    \label{integrate others}
    \resizebox{\textwidth}{!}{
\begin{tabular}{c|cccccc|cccccc}
\hline
                       & \multicolumn{2}{c}{CPQ}                                     & \multicolumn{2}{c}{BCQ-L}                                            & \multicolumn{2}{c|}{COptiDICE}                                              & \multicolumn{2}{c}{CPQ PROCO}                                                & \multicolumn{2}{c}{BCQ-L PROCO}                                      & \multicolumn{2}{c}{COptiDICE PROCO}                        \\
\multirow{-2}{*}{Task} & r↑                           & c↓                           & r↑                                   & c↓                                   & r↑                                   & c↓                                   & r↑                                    & c↓                                   & r↑                                   & c↓                                   & r↑                           & c↓                          \\ \hline
PointButton1           & {\color[HTML]{9B9B9B} 0.77}  & {\color[HTML]{9B9B9B} 13.37} & {\color[HTML]{9B9B9B} 0.25}          & {\color[HTML]{9B9B9B} 4.64}          & {\color[HTML]{9B9B9B} 0.06}          & {\color[HTML]{9B9B9B} 2.80}          & {\color[HTML]{9B9B9B} 0.23}           & {\color[HTML]{9B9B9B} 7.14}          & {\color[HTML]{9B9B9B} 0.07}          & {\color[HTML]{9B9B9B} 2.44}          & {\color[HTML]{9B9B9B} 0.14}  & {\color[HTML]{9B9B9B} 2.92} \\
PointButton2           & {\color[HTML]{9B9B9B} 0.62}  & {\color[HTML]{9B9B9B} 14.53} & {\color[HTML]{9B9B9B} 0.41}          & {\color[HTML]{9B9B9B} 8.65}          & {\color[HTML]{9B9B9B} 0.14}          & {\color[HTML]{9B9B9B} 4.86}          & {\color[HTML]{9B9B9B} 0.39}           & {\color[HTML]{9B9B9B} 12.07}         & {\color[HTML]{9B9B9B} 0.14}          & {\color[HTML]{9B9B9B} 2.83}          & {\color[HTML]{9B9B9B} 0.16}  & {\color[HTML]{9B9B9B} 4.42} \\
PointGoal1             & {\color[HTML]{9B9B9B} 0.80}  & {\color[HTML]{9B9B9B} 5.39}  & {\color[HTML]{9B9B9B} 0.71}          & {\color[HTML]{9B9B9B} 3.89}          & {\color[HTML]{9B9B9B} 0.50}          & {\color[HTML]{9B9B9B} 5.34}          & {\color[HTML]{9B9B9B} 0.30}           & {\color[HTML]{9B9B9B} 4.03}          & {\color[HTML]{9B9B9B} 0.82}          & {\color[HTML]{9B9B9B} 3.28}          & {\color[HTML]{9B9B9B} 0.52}  & {\color[HTML]{9B9B9B} 4.65} \\
PointGoal2             & {\color[HTML]{9B9B9B} 0.80}  & {\color[HTML]{9B9B9B} 18.69} & {\color[HTML]{9B9B9B} 0.66}          & {\color[HTML]{9B9B9B} 11.57}         & {\color[HTML]{9B9B9B} 0.38}          & {\color[HTML]{9B9B9B} 5.15}          & {\color[HTML]{9B9B9B} 0.09}           & {\color[HTML]{9B9B9B} 8.95}          & {\color[HTML]{9B9B9B} 0.26}          & {\color[HTML]{9B9B9B} 4.00}          & {\color[HTML]{9B9B9B} 0.34}  & {\color[HTML]{9B9B9B} 6.50} \\
PointPush1             & {\color[HTML]{9B9B9B} 0.19}  & {\color[HTML]{9B9B9B} 6.35}  & {\color[HTML]{9B9B9B} 0.31}          & {\color[HTML]{9B9B9B} 2.76}          & {\color[HTML]{9B9B9B} 0.09}          & {\color[HTML]{9B9B9B} 3.58}          & {\color[HTML]{9B9B9B} 0.14}           & {\color[HTML]{9B9B9B} 2.13}          & {\color[HTML]{9B9B9B} 0.24}          & {\color[HTML]{9B9B9B} 1.72}          & {\color[HTML]{9B9B9B} 0.12}  & {\color[HTML]{9B9B9B} 2.50} \\
PointPush2             & {\color[HTML]{9B9B9B} 0.22}  & {\color[HTML]{9B9B9B} 10.21} & {\color[HTML]{9B9B9B} 0.24}          & {\color[HTML]{9B9B9B} 4.03}          & {\color[HTML]{9B9B9B} 0.03}          & {\color[HTML]{9B9B9B} 3.00}          & {\color[HTML]{9B9B9B} 0.08}           & {\color[HTML]{9B9B9B} 2.94}          & {\color[HTML]{9B9B9B} 0.12}          & {\color[HTML]{9B9B9B} 2.86}          & {\color[HTML]{9B9B9B} -0.01} & {\color[HTML]{9B9B9B} 3.34} \\
CarButton1             & {\color[HTML]{9B9B9B} 0.35}  & {\color[HTML]{9B9B9B} 42.38} & {\color[HTML]{9B9B9B} -0.02}         & {\color[HTML]{9B9B9B} 5.56}          & {\color[HTML]{9B9B9B} -0.04}         & {\color[HTML]{9B9B9B} 3.69}          & {\color[HTML]{9B9B9B} 0.06}           & {\color[HTML]{9B9B9B} 7.08}          & {\color[HTML]{9B9B9B} -0.10}         & {\color[HTML]{9B9B9B} 2.42}          & {\color[HTML]{9B9B9B} -0.14} & {\color[HTML]{9B9B9B} 1.85} \\
CarButton2             & {\color[HTML]{9B9B9B} 0.54}  & {\color[HTML]{9B9B9B} 34.06} & {\color[HTML]{9B9B9B} 0.00}          & {\color[HTML]{9B9B9B} 5.15}          & {\color[HTML]{9B9B9B} -0.08}         & {\color[HTML]{9B9B9B} 2.93}          & {\color[HTML]{9B9B9B} 0.09}           & {\color[HTML]{9B9B9B} 10.62}         & {\color[HTML]{9B9B9B} -0.10}         & {\color[HTML]{9B9B9B} 3.73}          & {\color[HTML]{9B9B9B} -0.15} & {\color[HTML]{9B9B9B} 3.01} \\
CarGoal1               & {\color[HTML]{9B9B9B} 0.80}  & {\color[HTML]{9B9B9B} 5.74}  & {\color[HTML]{9B9B9B} 0.47}          & {\color[HTML]{9B9B9B} 2.86}          & {\color[HTML]{9B9B9B} 0.38}          & {\color[HTML]{9B9B9B} 2.14}          & {\color[HTML]{9B9B9B} 0.46}           & {\color[HTML]{9B9B9B} 3.04}          & {\color[HTML]{9B9B9B} 0.35}          & {\color[HTML]{9B9B9B} 2.36}          & {\color[HTML]{9B9B9B} 0.31}  & {\color[HTML]{9B9B9B} 1.73} \\
CarGoal2               & {\color[HTML]{9B9B9B} 0.82}  & {\color[HTML]{9B9B9B} 20.41} & {\color[HTML]{9B9B9B} 0.25}          & {\color[HTML]{9B9B9B} 3.67}          & {\color[HTML]{9B9B9B} 0.30}          & {\color[HTML]{9B9B9B} 3.72}          & {\color[HTML]{9B9B9B} 0.27}           & {\color[HTML]{9B9B9B} 6.12}          & {\color[HTML]{9B9B9B} 0.25}          & {\color[HTML]{9B9B9B} 3.30}          & {\color[HTML]{9B9B9B} 0.16}  & {\color[HTML]{9B9B9B} 1.59} \\
CarPush1               & {\color[HTML]{9B9B9B} -0.01} & {\color[HTML]{9B9B9B} 5.08}  & {\color[HTML]{9B9B9B} 0.20}          & {\color[HTML]{9B9B9B} 1.59}          & {\color[HTML]{9B9B9B} 0.20}          & {\color[HTML]{9B9B9B} 2.21}          & {\color[HTML]{9B9B9B} -0.13}          & {\color[HTML]{9B9B9B} 1.76}          & {\color[HTML]{9B9B9B} 0.19}          & {\color[HTML]{9B9B9B} 1.02}          & {\color[HTML]{9B9B9B} 0.22}  & {\color[HTML]{9B9B9B} 1.76} \\
CarPush2               & {\color[HTML]{9B9B9B} 0.29}  & {\color[HTML]{9B9B9B} 16.36} & {\color[HTML]{9B9B9B} 0.15}          & {\color[HTML]{9B9B9B} 6.94}          & {\color[HTML]{9B9B9B} 0.07}          & {\color[HTML]{9B9B9B} 3.00}          & {\color[HTML]{9B9B9B} 0.14}           & {\color[HTML]{9B9B9B} 2.73}          & {\color[HTML]{9B9B9B} 0.11}          & {\color[HTML]{9B9B9B} 3.40}          & {\color[HTML]{9B9B9B} 0.11}  & {\color[HTML]{9B9B9B} 2.95} \\
SwimmerVelocityV1      & {\color[HTML]{9B9B9B} 0.06}  & {\color[HTML]{9B9B9B} 11.65} & {\color[HTML]{9B9B9B} 0.56}          & {\color[HTML]{9B9B9B} 24.28}         & {\color[HTML]{9B9B9B} 0.48}          & {\color[HTML]{9B9B9B} 3.43}          & {\color[HTML]{3531FF} \textbf{-0.03}} & {\color[HTML]{3531FF} \textbf{0.05}} & {\color[HTML]{9B9B9B} 0.39}          & {\color[HTML]{9B9B9B} 3.20}          & {\color[HTML]{9B9B9B} 0.43}  & {\color[HTML]{9B9B9B} 4.35} \\
HopperVelocityV1       & {\color[HTML]{9B9B9B} 0.06}  & {\color[HTML]{9B9B9B} 5.09}  & {\color[HTML]{9B9B9B} 0.36}          & {\color[HTML]{9B9B9B} 7.66}          & {\color[HTML]{9B9B9B} 0.18}          & {\color[HTML]{9B9B9B} 4.59}          & {\color[HTML]{9B9B9B} 0.07}           & {\color[HTML]{9B9B9B} 4.57}          & {\color[HTML]{9B9B9B} 0.40}          & {\color[HTML]{9B9B9B} 5.77}          & {\color[HTML]{9B9B9B} 0.38}  & {\color[HTML]{9B9B9B} 2.73} \\
HalfCheetahVelocityV1  & {\color[HTML]{9B9B9B} 1.35}  & {\color[HTML]{9B9B9B} 88.45} & {\color[HTML]{9B9B9B} 1.01}          & {\color[HTML]{9B9B9B} 25.83}         & {\color[HTML]{3531FF} \textbf{0.59}} & {\color[HTML]{3531FF} \textbf{0.00}} & \textbf{0.18}                         & \textbf{0.01}                        & {\color[HTML]{9B9B9B} 0.91}          & {\color[HTML]{9B9B9B} 8.95}          & \textbf{0.58}                & \textbf{0.00}               \\
Walker2dVelocityV1     & {\color[HTML]{9B9B9B} 0.03}  & {\color[HTML]{9B9B9B} 1.08}  & {\color[HTML]{3531FF} \textbf{0.79}} & {\color[HTML]{3531FF} \textbf{0.24}} & {\color[HTML]{9B9B9B} 0.10}          & {\color[HTML]{9B9B9B} 1.68}          & \textbf{0.02}                         & \textbf{0.32}                        & \textbf{0.75}                        & \textbf{0.99}                        & {\color[HTML]{9B9B9B} 0.11}  & {\color[HTML]{9B9B9B} 2.06} \\
AntVelocityV1          & \textbf{-1.01}               & \textbf{0.00}                & {\color[HTML]{9B9B9B} 0.92}          & {\color[HTML]{9B9B9B} 27.89}         & {\color[HTML]{9B9B9B} 0.98}          & {\color[HTML]{9B9B9B} 5.16}          & \textbf{-0.92}                        & \textbf{0.00}                        & {\color[HTML]{3531FF} \textbf{0.84}} & {\color[HTML]{3531FF} \textbf{0.11}} & {\color[HTML]{9B9B9B} 0.97}  & {\color[HTML]{9B9B9B} 3.86} \\ \hline
Average                & {\color[HTML]{9B9B9B} 0.39}  & {\color[HTML]{9B9B9B} 17.58} & {\color[HTML]{9B9B9B} 0.43}          & {\color[HTML]{9B9B9B} 8.66}          & {\color[HTML]{9B9B9B} 0.26}          & {\color[HTML]{9B9B9B} 3.37}          & {\color[HTML]{9B9B9B} 0.09}           & {\color[HTML]{9B9B9B} 4.33}          & {\color[HTML]{9B9B9B} 0.33}          & {\color[HTML]{9B9B9B} 3.08}          & {\color[HTML]{9B9B9B} 0.25}  & {\color[HTML]{9B9B9B} 2.95} \\ \hline
Safe Percent           & \multicolumn{2}{c}{1/17}                                    & \multicolumn{2}{c}{1/17}                                                    & \multicolumn{2}{c|}{1/17}                                                   & \multicolumn{2}{c}{4/17}                                                     & \multicolumn{2}{c}{2/17}                                                    & \multicolumn{2}{c}{1/17}                                   \\ \hline
Safety Improvement     & \multicolumn{2}{c}{/}                                       & \multicolumn{2}{c}{/}                                                       & \multicolumn{2}{c|}{/}                                                      & \multicolumn{2}{c}{1325\%}                                                   & \multicolumn{2}{c}{558\%}                                                   & \multicolumn{2}{c}{42\%}                                   \\ \hline
\end{tabular}
    }
\end{table}

\subsection{Standard Deviation of the Main Results}
In this section, we present the standard deviation of \MethodName and various baselines computed with three different random seeds, as shown in \cref{detail std}.
The results demonstrate that variants of \MethodName achieve lower standard deviation compared to other baselines, indicating their stability.

\begin{table}[h!]
\small
\renewcommand{\arraystretch}{1.6}
    \centering
    \caption{Detailed standard deviation results of the main experiment. All results are computed using three different random seeds.}
    \label{detail std}
    \resizebox{\textwidth}{!}{
   \begin{tabular}{c|cccccccccccc|cccccc}
\hline
\multirow{2}{*}{Task} & \multicolumn{2}{c}{BC} & \multicolumn{2}{c}{CDT} & \multicolumn{2}{c}{CCAC} & \multicolumn{2}{c}{LSPC} & \multicolumn{2}{c}{CAPS} & \multicolumn{2}{c|}{FISOR} & \multicolumn{2}{c}{LSPC PROCO} & \multicolumn{2}{c}{CAPS PROCO} & \multicolumn{2}{c}{FISOR PROCO} \\
                      & r↑         & c↓        & r↑         & c↓         & r↑         & c↓          & r↑          & c↓         & r↑          & c↓         & r↑           & c↓          & r↑             & c↓            & r↑             & c↓            & r↑             & c↓             \\ \hline
PointButton1          & 0.05       & 1.21      & 0.14       & 1.90       & 0.00       & 1.76        & 0.01        & 0.86       & 0.11        & 1.87       & 0.12         & 0.91        & 0.04           & 0.32          & 0.02           & 0.34          & 0.01           & 0.56           \\
PointButton2          & 0.08       & 1.57      & 0.02       & 2.28       & 0.06       & 0.89        & 0.13        & 1.39       & 0.04        & 0.79       & 0.06         & 0.41        & 0.08           & 2.30          & 0.02           & 0.18          & 0.02           & 0.88           \\
PointGoal1            & 0.03       & 0.13      & 0.01       & 0.55       & 0.00       & 1.03        & 0.04        & 1.04       & 0.05        & 0.61       & 0.01         & 0.52        & 0.07           & 0.10          & 0.05           & 0.00          & 0.00           & 0.44           \\
PointGoal2            & 0.05       & 2.15      & 0.01       & 0.60       & 0.02       & 4.01        & 0.04        & 0.43       & 0.02        & 1.42       & 0.03         & 1.01        & 0.03           & 0.27          & 0.03           & 0.41          & 0.05           & 0.43           \\
PointPush1            & 0.03       & 0.91      & 0.03       & 0.15       & 0.04       & 1.66        & 0.01        & 1.11       & 0.03        & 0.84       & 0.06         & 1.27        & 0.01           & 1.38          & 0.02           & 0.33          & 0.00           & 0.53           \\
PointPush2            & 0.06       & 1.04      & 0.02       & 0.89       & 0.09       & 2.46        & 0.02        & 3.53       & 0.02        & 0.92       & 0.06         & 1.79        & 0.04           & 0.62          & 0.03           & 1.82          & 0.02           & 0.18           \\
CarButton1            & 0.01       & 1.35      & 0.02       & 1.86       & 0.05       & 24.41       & 0.07        & 1.53       & 0.08        & 3.29       & 0.03         & 1.78        & 0.02           & 2.43          & 0.01           & 1.30          & 0.04           & 0.23           \\
CarButton2            & 0.10       & 1.18      & 0.03       & 1.15       & 0.40       & 8.62        & 0.00        & 1.13       & 0.06        & 1.46       & 0.04         & 0.94        & 0.07           & 0.63          & 0.04           & 0.23          & 0.00           & 0.21           \\
CarGoal1              & 0.01       & 0.56      & 0.04       & 0.15       & 0.00       & 0.00        & 0.08        & 0.41       & 0.06        & 0.55       & 0.01         & 0.91        & 0.06           & 0.31          & 0.11           & 0.10          & 0.02           & 0.03           \\
CarGoal2              & 0.03       & 0.38      & 0.01       & 0.38       & 0.34       & 10.96       & 0.06        & 0.51       & 0.08        & 1.77       & 0.02         & 0.20        & 0.06           & 0.34          & 0.03           & 0.26          & 0.02           & 0.50           \\
CarPush1              & 0.05       & 0.24      & 0.01       & 1.13       & 0.33       & 0.17        & 0.01        & 0.53       & 0.01        & 0.82       & 0.04         & 0.35        & 0.02           & 0.74          & 0.03           & 0.09          & 0.02           & 0.86           \\
CarPush2              & 0.03       & 1.21      & 0.01       & 1.43       & 0.08       & 4.27        & 0.02        & 0.55       & 0.02        & 1.67       & 0.06         & 1.84        & 0.01           & 0.63          & 0.03           & 0.57          & 0.00           & 0.06           \\
SwimmerVelocityV1     & 0.10       & 5.05      & 0.02       & 3.13       & 0.05       & 3.57        & 0.04        & 4.55       & 0.13        & 6.07       & 0.01         & 0.56        & 0.06           & 0.59          & 0.03           & 0.12          & 0.02           & 0.12           \\
HopperVelocityV1      & 0.17       & 0.97      & 0.03       & 0.66       & 0.05       & 3.82        & 0.05        & 1.04       & 0.13        & 2.47       & 0.01         & 0.15        & 0.03           & 4.20          & 0.03           & 0.04          & 0.01           & 0.10           \\
HalfCheetahVelocityV1 & 0.00       & 2.94      & 0.01       & 1.01       & 0.07       & 0.65        & 0.00        & 3.84       & 0.01        & 2.85       & 0.01         & 3.45        & 0.05           & 0.00          & 0.01           & 0.00          & 0.01           & 0.00           \\
Walker2dVelocityV1    & 0.09       & 1.17      & 0.02       & 0.60       & 0.03       & 1.13        & 0.01        & 0.00       & 0.02        & 0.53       & 0.03         & 0.95        & 0.08           & 1.31          & 0.09           & 0.78          & 0.03           & 0.09           \\
AntVelocityV1         & 0.01       & 3.91      & 0.00       & 0.31       & 0.00       & 0.00        & 0.00        & 1.48       & 0.00        & 0.31       & 0.00         & 0.76        & 0.01           & 0.21          & 0.06           & 0.00          & 0.04           & 0.00           \\ \hline
Average               & 0.05       & 1.53      & 0.03       & 1.07       & 0.09       & 4.08        & 0.03        & 1.41       & 0.05        & 1.66       & 0.04         & 1.05        & 0.04           & 0.96          & 0.04           & 0.39          & 0.02           & 0.31           \\ \hline
\end{tabular}
    }
\end{table}

\subsection{Detailed ablation results}
In this section, we also provide the detailed results of the ablation studies, as shown in \cref{detail abl}.
\begin{table}[h!]
\small
\renewcommand{\arraystretch}{1.0}
    \centering
    \caption{Detailed ablation results (mean$\pm$std).}
    \label{detail abl}
    \resizebox{\textwidth}{!}{

\begin{tabular}{c|cccccc}
\hline
\multirow{2}{*}{Task} & \multicolumn{2}{c}{W/o Model}  & \multicolumn{2}{c}{Full Model} & \multicolumn{2}{c}{W/o Relabel} \\
                      & r↑             & c↓            & r↑             & c↓            & r↑              & c↓            \\ \hline
PointButton1          & 0.06$\pm$0.00  & 2.84$\pm$0.58 & 0.01$\pm$0.00  & 1.82$\pm$0.76 & 0.07$\pm$0.01   & 1.10$\pm$0.39 \\
PointButton2          & 0.09$\pm$0.03  & 4.32$\pm$0.28 & 0.03$\pm$0.02  & 2.14$\pm$0.72 & 0.09$\pm$0.02   & 2.19$\pm$0.36 \\
PointGoal1            & 0.54$\pm$0.04  & 2.41$\pm$0.76 & 0.07$\pm$0.02  & 0.18$\pm$0.13 & 0.28$\pm$0.04   & 0.25$\pm$0.13 \\
PointGoal2            & 0.14$\pm$0.03  & 1.58$\pm$0.68 & 0.03$\pm$0.01  & 0.63$\pm$0.34 & 0.09$\pm$0.05   & 0.92$\pm$0.30 \\
PointPush1            & 0.25$\pm$0.02  & 1.03$\pm$0.32 & 0.13$\pm$0.04  & 0.79$\pm$0.68 & 0.21$\pm$0.02   & 0.60$\pm$0.32 \\
PointPush2            & 0.11$\pm$0.03  & 1.21$\pm$0.36 & 0.04$\pm$0.05  & 0.48$\pm$0.25 & 0.07$\pm$0.03   & 1.22$\pm$0.27 \\
CarButton1            & -0.02$\pm$0.01 & 1.24$\pm$0.25 & 0.00$\pm$0.00  & 1.57$\pm$1.22 & -0.03$\pm$0.03  & 1.22$\pm$0.37 \\
CarButton2            & 0.00$\pm$0.02  & 3.22$\pm$0.94 & -0.06$\pm$0.01 & 1.99$\pm$0.80 & -0.03$\pm$0.03  & 1.76$\pm$1.04 \\
CarGoal1              & 0.27$\pm$0.02  & 0.37$\pm$0.13 & 0.02$\pm$0.02  & 0.00$\pm$0.00 & 0.15$\pm$0.06   & 0.31$\pm$0.18 \\
CarGoal2              & 0.06$\pm$0.03  & 0.14$\pm$0.12 & -0.01$\pm$0.02 & 0.07$\pm$0.07 & 0.11$\pm$0.02   & 1.00$\pm$0.34 \\
CarPush1              & 0.24$\pm$0.04  & 1.03$\pm$0.69 & 0.16$\pm$0.03  & 0.12$\pm$0.17 & 0.17$\pm$0.02   & 0.29$\pm$0.19 \\
CarPush2              & 0.03$\pm$0.03  & 0.48$\pm$0.03 & -0.03$\pm$0.05 & 0.00$\pm$0.00 & 0.06$\pm$0.02   & 1.27$\pm$0.53 \\
SwimmerVelocityV1     & 0.01$\pm$0.01  & 0.00$\pm$0.00 & 0.02$\pm$0.00  & 0.00$\pm$0.00 & 0.01$\pm$0.02   & 0.07$\pm$0.10 \\
HopperVelocityV1      & 0.17$\pm$0.05  & 0.17$\pm$0.12 & 0.04$\pm$0.03  & 0.00$\pm$0.00 & 0.12$\pm$0.09   & 0.18$\pm$0.20 \\
HalfCheetahVelocityV1 & 0.51$\pm$0.04  & 0.00$\pm$0.00 & 0.23$\pm$0.05  & 0.00$\pm$0.00 & 0.83$\pm$0.04   & 1.40$\pm$1.89 \\
Walker2dVelocityV1    & 0.15$\pm$0.03  & 1.47$\pm$0.50 & 0.05$\pm$0.03  & 0.01$\pm$0.01 & 0.18$\pm$0.05   & 3.36$\pm$1.38 \\
AntVelocityV1         & 0.59$\pm$0.01  & 0.00$\pm$0.00 & 0.00$\pm$0.08  & 0.01$\pm$0.01 & 0.52$\pm$0.05   & 0.01$\pm$0.01 \\ \hline
Average               & 0.19           & 1.27          & 0.05           & 0.58          & 0.17            & 1.01          \\ \hline
\end{tabular}
}

\vspace{1em}

\resizebox{\textwidth}{!}{
\begin{tabular}{c|cccccc}
\hline
\multirow{2}{*}{Task} & \multicolumn{2}{c}{Det. Rollout} & \multicolumn{2}{c}{W/o Consv.} & \multicolumn{2}{c}{W/o Refl.}  \\
                      & r↑              & c↓             & r↑             & c↓            & r↑             & c↓            \\ \hline
PointButton1          & 0.02$\pm$0.04   & 1.67$\pm$0.97  & 0.09$\pm$0.03  & 4.01$\pm$0.27 & 0.06$\pm$0.03  & 3.40$\pm$1.70 \\
PointButton2          & 0.07$\pm$0.01   & 2.90$\pm$0.64  & 0.08$\pm$0.01  & 3.54$\pm$0.80 & 0.07$\pm$0.01  & 4.55$\pm$1.48 \\
PointGoal1            & 0.40$\pm$0.02   & 1.65$\pm$0.36  & 0.37$\pm$0.00  & 0.63$\pm$0.21 & 0.34$\pm$0.04  & 0.86$\pm$0.55 \\
PointGoal2            & 0.06$\pm$0.01   & 0.25$\pm$0.18  & 0.06$\pm$0.01  & 0.74$\pm$0.54 & 0.08$\pm$0.02  & 0.42$\pm$0.13 \\
PointPush1            & 0.18$\pm$0.04   & 0.91$\pm$1.10  & 0.22$\pm$0.02  & 0.94$\pm$0.14 & 0.20$\pm$0.05  & 2.17$\pm$2.38 \\
PointPush2            & 0.08$\pm$0.04   & 0.51$\pm$0.38  & 0.11$\pm$0.04  & 3.98$\pm$1.95 & 0.07$\pm$0.01  & 1.22$\pm$0.90 \\
CarButton1            & -0.04$\pm$0.04  & 1.75$\pm$0.59  & -0.04$\pm$0.02 & 2.76$\pm$1.10 & -0.01$\pm$0.03 & 2.57$\pm$1.32 \\
CarButton2            & -0.01$\pm$0.01  & 1.36$\pm$0.27  & -0.01$\pm$0.00 & 2.40$\pm$0.72 & 0.00$\pm$0.01  & 3.21$\pm$1.44 \\
CarGoal1              & 0.15$\pm$0.04   & 0.03$\pm$0.04  & 0.23$\pm$0.04  & 0.50$\pm$0.37 & 0.16$\pm$0.02  & 0.16$\pm$0.14 \\
CarGoal2              & 0.03$\pm$0.00   & 0.08$\pm$0.11  & 0.03$\pm$0.00  & 0.39$\pm$0.49 & 0.04$\pm$0.03  & 0.80$\pm$0.93 \\
CarPush1              & 0.18$\pm$0.03   & 1.87$\pm$1.23  & 0.22$\pm$0.03  & 0.45$\pm$0.24 & 0.11$\pm$0.04  & 1.86$\pm$2.43 \\
CarPush2              & 0.00$\pm$0.00   & 0.04$\pm$0.05  & 0.04$\pm$0.03  & 3.00$\pm$1.34 & 0.03$\pm$0.04  & 1.19$\pm$0.97 \\
SwimmerVelocityV1     & 0.00$\pm$0.01   & 0.03$\pm$0.02  & 0.01$\pm$0.01  & 0.00$\pm$0.00 & 0.00$\pm$0.00  & 0.23$\pm$0.33 \\
HopperVelocityV1      & 0.20$\pm$0.12   & 0.01$\pm$0.01  & 0.06$\pm$0.08  & 0.00$\pm$0.00 & 0.13$\pm$0.17  & 0.00$\pm$0.00 \\
HalfCheetahVelocityV1 & 0.49$\pm$0.02   & 0.00$\pm$0.00  & 0.81$\pm$0.01  & 0.00$\pm$0.00 & 0.65$\pm$0.07  & 0.00$\pm$0.00 \\
Walker2dVelocityV1    & 0.15$\pm$0.01   & 1.57$\pm$0.36  & 0.22$\pm$0.04  & 2.28$\pm$0.15 & 0.18$\pm$0.00  & 2.09$\pm$0.57 \\
AntVelocityV1         & 0.46$\pm$0.05   & 0.00$\pm$0.00  & 0.41$\pm$0.03  & 0.00$\pm$0.00 & 0.50$\pm$0.18  & 0.00$\pm$0.00 \\ \hline
Average               & 0.14            & 0.86           & 0.17           & 1.51          & 0.15           & 1.45          \\ \hline
\end{tabular}
    }
\end{table}

\subsection{Detailed Sensitivity Analysis Results}
Finally, in this section, we provide the detailed results of sensitivity studies, as shown in \cref{detail sens}.
Notably, while safety performance degrades significantly when $H=5$, this degradation is mainly observed in the Point and Car tasks, whereas the MuJoCo Velocity tasks do not exhibit such a decline and may even become more conservative. These results suggest that when the environment model is sufficiently accurate, increasing $H$ can indeed yield improved safety performance; however, if the environment model accuracy is low, keeping $H$ small is preferable to avoid the negative impact of model errors on learning stability.
\begin{table}[h!]
\small
\renewcommand{\arraystretch}{1.0}
    \centering
    \caption{Detailed sensitivity analysis results (mean$\pm$std).}
    \label{detail sens}
    \resizebox{\textwidth}{!}{

\begin{tabular}{c|cccccc}
\hline
\multirow{2}{*}{Task} & \multicolumn{2}{c}{$E=2$}      & \multicolumn{2}{c}{$E=5$}      & \multicolumn{2}{c}{$H=2$}      \\
                      & r↑             & c↓            & r↑             & c↓            & r↑             & c↓            \\ \hline
PointButton1          & 0.03$\pm$0.01  & 1.93$\pm$0.14 & 0.03$\pm$0.00  & 1.55$\pm$0.55 & -0.02$\pm$0.01 & 0.57$\pm$0.20 \\
PointButton2          & 0.06$\pm$0.04  & 2.31$\pm$1.15 & 0.06$\pm$0.01  & 2.42$\pm$0.58 & -0.06$\pm$0.05 & 1.79$\pm$0.33 \\
PointGoal1            & 0.45$\pm$0.03  & 1.70$\pm$1.03 & 0.42$\pm$0.07  & 1.28$\pm$0.66 & 0.14$\pm$0.03  & 2.67$\pm$0.31 \\
PointGoal2            & 0.10$\pm$0.04  & 0.71$\pm$0.45 & 0.08$\pm$0.02  & 0.53$\pm$0.31 & -0.04$\pm$0.04 & 1.20$\pm$0.27 \\
PointPush1            & 0.20$\pm$0.02  & 0.43$\pm$1.35 & 0.18$\pm$0.02  & 1.14$\pm$0.53 & 0.13$\pm$0.05  & 1.10$\pm$0.44 \\
PointPush2            & 0.10$\pm$0.00  & 0.76$\pm$0.31 & 0.09$\pm$0.02   & 0.24$\pm$0.16 & 0.06$\pm$0.02  & 0.65$\pm$0.25 \\
CarButton1            & -0.04$\pm$0.05 & 0.61$\pm$0.42 & -0.07$\pm$0.05 & 0.31$\pm$0.11 & -0.04$\pm$0.03 & 0.94$\pm$0.04 \\
CarButton2            & -0.02$\pm$0.04 & 1.85$\pm$0.29 & -0.03$\pm$0.03 & 1.68$\pm$0.43 & -0.13$\pm$0.06 & 0.75$\pm$0.27 \\
CarGoal1              & 0.22$\pm$0.02  & 0.15$\pm$0.10 & 0.20$\pm$0.02  & 0.03$\pm$0.05 & 0.16$\pm$0.08  & 1.45$\pm$0.59 \\
CarGoal2              & 0.05$\pm$0.02  & 0.58$\pm$0.26 & 0.05$\pm$0.00  & 0.77$\pm$0.11 & 0.05$\pm$0.04  & 2.29$\pm$0.89 \\
CarPush1              & 0.19$\pm$0.04  & 0.35$\pm$0.22 & 0.17$\pm$0.03  & 0.45$\pm$0.39 & 0.12$\pm$0.02  & 0.32$\pm$0.09 \\
CarPush2              & 0.03$\pm$0.01  & 0.01$\pm$0.01 & -0.01$\pm$0.02 & 0.00$\pm$0.00 & -0.01$\pm$0.04 & 0.64$\pm$0.91 \\
SwimmerVelocityV1     & 0.01$\pm$0.00  & 0.11$\pm$0.15 & 0.02$\pm$0.01  & 0.32$\pm$0.45 & 0.01$\pm$0.01  & 0.00$\pm$0.00 \\
HopperVelocityV1      & 0.18$\pm$0.02  & 0.21$\pm$0.14 & 0.38$\pm$0.15  & 0.03$\pm$0.03 & 0.13$\pm$0.04  & 0.14$\pm$0.19 \\
HalfCheetahVelocityV1 & 0.50$\pm$0.04  & 0.00$\pm$0.00 & 0.49$\pm$0.03  & 0.00$\pm$0.00 & 0.39$\pm$0.07  & 0.00$\pm$0.00 \\
Walker2dVelocityV1    & 0.13$\pm$0.03  & 1.50$\pm$0.40 & 0.11$\pm$0.04  & 1.26$\pm$0.59 & 0.05$\pm$0.03  & 0.25$\pm$0.09 \\
AntVelocityV1         & 0.59$\pm$0.00  & 0.00$\pm$0.00 & 0.53$\pm$0.01  & 0.00$\pm$0.00 & 0.29$\pm$0.05  & 0.00$\pm$0.00 \\ \hline
Average               & 0.16           & 0.77          & 0.16           & 0.70          & 0.07           & 0.87          \\ \hline
\end{tabular}
}

\vspace{1em}

\resizebox{\textwidth}{!}{
\begin{tabular}{c|cccccc}
\hline
\multirow{2}{*}{Task} & \multicolumn{2}{c}{$H=5$}       & \multicolumn{2}{c}{Deepseek R1} & \multicolumn{2}{c}{Gemini 2.5 Pro} \\
                      & r↑             & c↓             & r↑              & c↓            & r↑               & c↓              \\ \hline
PointButton1          & 0.21$\pm$0.13  & 10.36$\pm$1.13 & 0.05$\pm$0.01   & 2.31$\pm$0.57 & 0.01$\pm$0.02    & 1.18$\pm$0.64   \\
PointButton2          & 0.27$\pm$0.07  & 11.76$\pm$2.04 & 0.08$\pm$0.03   & 2.57$\pm$0.60 & 0.07$\pm$0.02    & 1.80$\pm$0.27   \\
PointGoal1            & 0.06$\pm$0.13  & 5.19$\pm$3.41  & 0.35$\pm$0.03   & 0.71$\pm$0.24 & 0.31$\pm$0.02    & 0.63$\pm$0.20   \\
PointGoal2            & 0.26$\pm$0.06  & 6.49$\pm$3.44  & 0.05$\pm$0.03   & 0.11$\pm$0.08 & 0.10$\pm$0.04    & 0.68$\pm$0.31   \\
PointPush1            & 0.11$\pm$0.02  & 1.04$\pm$0.99  & 0.25$\pm$0.07   & 1.79$\pm$1.42 & 0.17$\pm$0.02    & 1.04$\pm$0.39   \\
PointPush2            & 0.10$\pm$0.02  & 1.69$\pm$2.18  & 0.12$\pm$0.02   & 0.82$\pm$0.48 & 0.10$\pm$0.01    & 0.51$\pm$0.35   \\
CarButton1            & -0.01$\pm$0.02 & 6.75$\pm$8.33  & -0.03$\pm$0.04  & 1.04$\pm$0.24 & -0.03$\pm$0.01   & 0.60$\pm$0.24   \\
CarButton2            & 0.03$\pm$0.01  & 13.28$\pm$4.62 & 0.00$\pm$0.01   & 1.94$\pm$0.34 & -0.03$\pm$0.02   & 1.02$\pm$0.21   \\
CarGoal1              & 0.16$\pm$0.09  & 3.11$\pm$2.03  & 0.15$\pm$0.04   & 0.11$\pm$0.06 & 0.11$\pm$0.00    & 0.03$\pm$0.04   \\
CarGoal2              & 0.06$\pm$0.04  & 8.77$\pm$7.05  & 0.02$\pm$0.02   & 0.60$\pm$0.33 & 0.03$\pm$0.01    & 0.38$\pm$0.42   \\
CarPush1              & 0.18$\pm$0.00  & 1.57$\pm$1.23  & 0.22$\pm$0.07   & 0.85$\pm$0.47 & 0.13$\pm$0.02    & 0.49$\pm$0.08   \\
CarPush2              & 0.06$\pm$0.01  & 4.16$\pm$1.62  & 0.07$\pm$0.05   & 0.42$\pm$0.06 & 0.04$\pm$0.02    & 0.15$\pm$0.21   \\
SwimmerVelocityV1     & 0.04$\pm$0.02  & 0.00$\pm$0.00  & -0.01$\pm$0.01  & 0.00$\pm$0.00 & 0.00$\pm$0.03    & 0.00$\pm$0.00   \\
HopperVelocityV1      & 0.06$\pm$0.05  & 0.47$\pm$0.35  & 0.29$\pm$0.10   & 1.27$\pm$1.39 & 0.20$\pm$0.03    & 0.25$\pm$0.27   \\
HalfCheetahVelocityV1 & 0.42$\pm$0.07  & 0.00$\pm$0.00  & 0.08$\pm$0.02   & 0.00$\pm$0.00 & 0.54$\pm$0.01    & 0.00$\pm$0.00   \\
Walker2dVelocityV1    & 0.12$\pm$0.04  & 0.43$\pm$0.26  & 0.17$\pm$0.02   & 2.03$\pm$0.46 & 0.15$\pm$0.02    & 1.66$\pm$0.61   \\
AntVelocityV1         & 0.49$\pm$0.01  & 0.00$\pm$0.00  & 0.54$\pm$0.11   & 0.00$\pm$0.00 & 0.39$\pm$0.10    & 0.00$\pm$0.00   \\ \hline
Average               & 0.15           & 4.41           & 0.14            & 0.97          & 0.14             & 0.61            \\ \hline
\end{tabular}
    }
\end{table}

\section{Impact Statement}
The goal of this work is to advance offline safe RL under safe-only datasets, thereby enabling the deployment of learned policies in a broader range of real-world applications. For example, in embodied intelligence, data are often collected via human teleoperation; when an imminent unsafe collision is detected, data collection may be interrupted, or external interventions may alter the agent’s state or the environment to prevent collisions. The resulting datasets naturally contain only safe trajectories, yet they are highly valuable for training safe policies. Our method is specifically designed to operate under such realistic data collection paradigms.
Furthermore, the work presented in this paper does not introduce any additional ethical concerns beyond those commonly associated with RL research. Therefore, no special discussion of ethical issues is required.

\section{Discussions and Limitations}
In this work, although LLMs are employed as the primary approach for cost function generation, our core objective is not to rely on LLMs per se. Rather, we view LLMs as a knowledge processing tool that transforms natural-language safety specifications into a cost function—an intermediate representation that can be effectively aligned with data. Consequently, alternative approaches that serve a similar role can also be naturally used. For instance, human experts could directly process and translate safety descriptions into a cost function, which can then be iteratively refined based on feedback from the data. Compared to manual efforts, the use of LLMs offers a higher degree of automation and reduced human labor, while a potential drawback lies in the additional computational and monetary costs associated with LLM inference.

Meanwhile, although this work effectively addresses the challenge of learning offline safe policies in scenarios with scarce or no unsafe data, some limitations remain. The first lies in the accuracy and generalization of the environment model. When model accuracy is low, the rollout horizon must be restricted. A potential solution is to employ more powerful generative models, such as diffusion models, as environment models. The second limitation concerns the generation of conservative cost functions: the cost description must be grounded in features observable in the agent’s observations. If certain variables required to compute the cost function cannot be directly or indirectly inferred from the observations, our approach cannot be applied. 
Additionally, despite the adoption of a check-and-feedback mechanism, the reliability of LLM outputs remains affected by hallucinations. A promising direction for future research is to integrate human-in-the-loop strategies to mitigate this issue.
Finally, due to the reliance on cost function generation, the current \MethodName is restricted to state-based tasks and cannot be applied to vision-based tasks. This is because VLMs cannot directly produce sufficiently accurate image-based cost functions, and performing VLM-based evaluation for every sample would be prohibitively expensive.


\end{document}